\documentclass{article}

    \usepackage[preprint]{neurips_2025}

\usepackage[utf8]{inputenc} 
\usepackage[T1]{fontenc}    
\usepackage{hyperref}       
\usepackage{url}            
\usepackage{booktabs}       
\usepackage{amsfonts}       
\usepackage{nicefrac}       
\usepackage{microtype}      
\usepackage{xcolor}         
\usepackage{multirow}
\usepackage[linesnumbered,ruled,vlined]{algorithm2e}

\usepackage{bbm}
\usepackage{amsmath,amssymb,amsthm,bm}
\usepackage{enumitem}
\usepackage{graphicx}

\theoremstyle{plain}
\newtheorem{theorem}{Theorem}[section]
\newtheorem{assumption}[theorem]{Assumption}

\newtheorem{corollary}[theorem]{Corollary}

\theoremstyle{definition}

\title{Reward Model Generalization for Compute-Aware Test-Time Reasoning}

\author{%
  Zeen~Song \\
  Institute of Software, \\
  China Academy of Sciences\\
  University of Chinese Academy of Science\\
  \And
  Wenwen Qiang \\
  Institute of Software,\\
  China Academy of Sciences
  \And
  Siyu Zhao \\
  University of Chinese Academy of Science
  \And
  Changwen Zheng\\
  Institute of Software,\\
  China Academy of Sciences
  \And
  Gang Hua \\
  Multimodal Experiences Lab, \\
  Dolby Laboratories Inc\\
  Institute of Artificial Intelligence and Robotics, \\
  Xi’an Jiaotong University
}

\begin{document}

\maketitle

\begin{abstract}
  External test-time reasoning enhances large language models (LLMs) by decoupling generation and selection. At inference time, the model generates multiple reasoning paths, and an auxiliary process reward model (PRM) is used to score and select the best one. A central challenge in this setting is test-time compute optimality (TCO), i.e., how to maximize answer accuracy under a fixed inference budget. In this work, we establish a theoretical framework to analyze how the generalization error of the PRM affects compute efficiency and reasoning performance. Leveraging PAC-Bayes theory, we derive generalization bounds and show that a lower generalization error of PRM leads to fewer samples required to find correct answers. Motivated by this analysis, we propose Compute-Aware Tree Search (CATS), an actor-critic framework that dynamically controls search behavior. The actor outputs sampling hyperparameters based on reward distributions and sparsity statistics, while the critic estimates their utility to guide budget allocation. Experiments on the MATH and AIME benchmarks with various LLMs and PRMs demonstrate that CATS consistently outperforms other external TTS methods, validating our theoretical predictions.
\end{abstract}

\section{Introduction}

In recent years, chain-of-thought (CoT) prompting has substantially improved the performance of large language models (LLMs) on complex reasoning tasks such as math problem solving, question answering, and multi-hop retrieval \cite{weiChainofthoughtPromptingElicits2022, kojimaLargeLanguageModels2022, yaoTreeThoughtsDeliberate2023}. To support more effective CoT reasoning, recent works have explored test-time scaling (TTS) strategies that allocate more compute during inference \cite{openai2024learning, deepseek-aiDeepSeekR1IncentivizingReasoning2025, teamKimiK15Scaling2025,snellScalingLLMTestTime2024a, liuCan1BLLM2025, wuInferenceScalingLaws2024}. These approaches can be broadly divided into internal and external ones, with external methods attracting increasing attention due to their flexibility and ability to enhance performance without modifying the base model \cite{snellScalingLLMTestTime2024a, liuCan1BLLM2025, wuInferenceScalingLaws2024}.


The external TTS framework typically consists of three components: a frozen policy model, a scaling algorithm, and a process reward model (PRM). It generates multiple candidate reasoning paths using the policy model guided by the scaling algorithm (e.g., Best-of-N or Beam Search), and reranks these paths using the PRM to select the most promising one \cite{yaoTreeThoughtsDeliberate2023, bestaGraphThoughtsSolving2024}.  A central challenge in this setting is test-time compute optimality (TCO), which aims to select the optimal hyperparameters that maximize answer accuracy with a fixed policy model and given compute budget \cite{snellScalingLLMTestTime2024a, liuCan1BLLM2025, quOptimizingTestTimeCompute2025}. Empirical studies suggest that among these components, the PRM often plays a critical role in determining TCO performance \cite{liuCan1BLLM2025}. However, most PRMs are trained via supervised learning on limited datasets \cite{lightmanLetsVerifyStep2023, wangMathShepherdVerifyReinforce2024}, and their ability to generalize to unseen reasoning paths significantly affects the accuracy of path selection. Despite their growing importance, the effects of the generalization error of PRM on downstream reasoning performance remain underexplored.

To systematically characterize the role of PRMs in achieving TCO, we propose a unified theoretical framework. We first show that the generalization error of the PRM is upper-bounded via PAC-Bayesian analysis. Then, we establish the relationship between the generalization error, the final answer accuracy, and the available computing budget. One of our key insights is to quantify the risk of mis-ranking candidate reasoning paths due to generalization error. Our analysis shows that the answer accuracy of external TTS is lower-bounded by three components: \textbf{(i)} the probability that the policy model generates a correct answer, \textbf{(ii)} the \textbf{reward gap} between the selected path and the discarded ones, and \textbf{(iii)} the upper bound of the generalization error.  

While our theoretical framework provides a target for optimizing inference-time accuracy, it also reveals two key practical challenges. First, the generalization error of the PRM significantly affects answer accuracy under a fixed compute budget, yet it is unobservable at test time. Second, although the reward gap can be influenced by tuning sampling parameters such as top-$k$, top-$p$, and temperature, its effect varies across different PRMs. To address these challenges, we use parameter sparsity as a proxy for the generalization error of the PRM and propose Compute-Aware Tree Search (CATS), a dynamic reasoning control strategy based on the Advantage Actor-Critic (A2C) framework. CATS formulates the inference process as a Markov Decision Process (MDP) and controls the search configuration by learning an actor network. At each step, the actor network outputs a search configuration based on candidate rewards and model sparsity, while the critic estimates the value of the current state via a temporal-difference (TD) objective. By jointly training on multiple PRMs, CATS learns to adaptively adjust the number and selection of candidate reasoning paths, effectively improving the answer accuracy under limited compute.

To validate our theoretical analysis and the effectiveness of the proposed CATS strategy, we conduct extensive experiments on two challenging mathematical reasoning benchmarks: MATH500 \cite{hendrycksmath2021, lightmanLetsVerifyStep2023} and AIME24 \cite{ai-mo2024}. We evaluate CATS under multiple frozen policy models, including Qwen 2.5 \cite{yang2024qwen2}, Llama 3.1 \cite{grattafiori2024llama}, and Llama 3.2 \cite{meta2024llama3}, and incorporate a diverse set of PRMs \cite{zhangLessonsDevelopingProcess2025, wangMathShepherdVerifyReinforce2024, dong2024rlhf, skywork2024}. The results demonstrate that CATS consistently achieves higher accuracy than other external TTS methods across different model combinations and compute budgets. These results provide strong empirical support for our theoretical predictions. Our contributions can be summarized as follows:
\begin{itemize}
    \item We present a unified theoretical framework that establishes a quantitative relationship between the generalization error of the PRM, compute budget, and answer accuracy in external TTS. By analyzing the risk of mis-ranking candidate reasoning paths, we derive an explicit lower bound on answer accuracy in terms of the reward gap, sampling coverage, and the generalization error. 
    \item Motivated by our theoretical analysis, we propose CATS, a reasoning control strategy based on the A2C framework. CATS dynamically allocates compute across inference steps by adjusting path selection and generation parameters, using sparsity as a proxy signal for the generalization error of PRM.
    \item We evaluate CATS across diverse challenging reasoning benchmarks, different policy models, and several PRMs. Results show that CATS consistently improves accuracy across all settings, validating our theoretical predictions.
\end{itemize}

\section{Related Works}

\textbf{Scaling of Test-time Compute}. CoT prompting is first proposed as a prompting technique that enables LLMs to decompose problems into intermediate steps \cite{weiChainofthoughtPromptingElicits2022}. Recently, the OpenAI o1 series \cite{openai2024learning} demonstrate that increasing the length of CoT during inference yields substantial performance gains on tasks like MATH \cite{hendrycksmath2021} and AIME \cite{ai-mo2024}. 
TTS approaches can be broadly categorized into internal and external methods \cite{snellScalingLLMTestTime2024a,liuCan1BLLM2025,chenReasoningEraSurvey2025}. Internal TTS encourages models to extend CoT reasoning via supervised fine-tuning (SFT) or reinforcement learning (RL). Some methods construct training data to promote step-wise self-refinement \cite{madaanSelfRefineIterativeRefinement2023,saunders2022self}. 
DeepSeek-R1 \cite{deepseek-aiDeepSeekR1IncentivizingReasoning2025} combines formatting-based and rule-based rewards, and optimizes the model using GRPO \cite{shaoDeepSeekMathPushingLimits2024}.
In contrast, external TTS improves the reasoning performance via sampling or search-based methods with fixed LLMs and an external verifier \cite{lightmanLetsVerifyStep2023,wuInferenceScalingLaws2024, snellScalingLLMTestTime2024a, yaoTreeThoughtsDeliberate2023, selAlgorithmThoughtsEnhancing2024, bestaGraphThoughtsSolving2024, zhangAutomaticChainThought2022, brownLargeLanguageMonkeys2024}. Specifically, Tree‑of‑Thoughts \cite{yaoTreeThoughtsDeliberate2023} explores a look‑ahead search tree of thought chunks with self‑evaluation to achieve large gains on planning‑style tasks. 
Snell et al. \cite{snellScalingLLMTestTime2024a} analyze compute‑optimal test‑time scaling, finding that adaptive allocation of verifier‑guided search can beat a 14× larger model while using 4× less extra compute than best‑of‑N sampling.


\textbf{Process Reward Model}.
An essential component of external TTS is the verifier that evaluates different reasoning paths. Verifiers are categorized into two types: Process Reward Models (PRMs) and Outcome Reward Models (ORMs) \cite{uesato2022solving}. PRMs assess the quality of a reasoning step given the question and partial reasoning trajectory, estimating the likelihood that the process will lead to a correct answer \cite{lightmanLetsVerifyStep2023, wangMathShepherdVerifyReinforce2024}. In contrast, ORMs provide a reward signal based on the final answer's correctness, given the full reasoning trace and output \cite{uesato2022solving,lightmanLetsVerifyStep2023}. Recent studies have shown that PRMs are generally more effective than ORMs in guiding search \cite{lightmanLetsVerifyStep2023, uesato2022solving, snellScalingLLMTestTime2024a}, and PRMs have become a widely adopted tool in external test-time reasoning frameworks \cite{liuCan1BLLM2025, xieSelfEvaluationGuidedBeam2023,snellScalingLLMTestTime2024a}. Lightman et al \cite{lightmanLetsVerifyStep2023} trains PRM on 800k human‑labeled reasoning steps. Math-Shepherd \cite{wangMathShepherdVerifyReinforce2024} automatically constructs step-level supervision by forward decoding multiple reasoning branches from each intermediate step and assigning scores based on the proportion of branches that reach the known correct answer.

\section{Problem Formulation}
In this section, we first formalize the reasoning task and briefly introduce external TTS methods. We highlight TCO as a central objective in this setting. We then introduce PRM as a key component of the external TTS framework and describe its training process. Finally, we present the central motivation of this work: understanding how the generalization performance of PRM affects TCO. 
\subsection{The Problem Definition of Reasoning and Scaling of Test-time Compute}

Given an input question \(q\in\mathcal{Q}\), the reasoning problem can be presented as outputting an answer \(a\in\mathcal{A}\) that matches the ground-truth answer \(a^*(q)\in\mathcal{A}\), using a policy model \(\pi_\theta\). Here, \(\mathcal{Q}\) and \(\mathcal{A}\) denote the sample spaces of questions and answers, respectively, and \(\pi_\theta\) is a pre-trained LLM. To tackle this challenge, a promising direction is to scale the test-time compute by allocating more computational resources during answer generation. Specifically, the policy model generates a CoT reasoning path \(h = (z_1, z_2, \dots, z_T)\) in an autoregressive manner, where each step \(z_t\) is sampled from \(\pi_\theta(\cdot \mid q, z_1, \dots, z_{t-1})\), followed by sampling the final answer \(a\sim\pi_\theta(\cdot\mid z_1, \dots, z_T)\). This class of methods is also referred to as TTS methods. Existing TTS methods can be broadly categorized into two types: internal TTS and external TTS \cite{snellScalingLLMTestTime2024a,liuCan1BLLM2025,chenReasoningEraSurvey2025}. While internal TTS modifies \(\pi_\theta\) via fine-tuning to encourage longer reasoning paths for complex problems \cite{openai2024learning,deepseek-aiDeepSeekR1IncentivizingReasoning2025}, in this work, we primarily focus on the external TTS methods, which samples a collection of reasoning paths \(\mathcal{H} = \{h_1, h_2, \dots, h_N\}\) and employs an external PRM \(R_\phi\) to score each path $h\in\mathcal{H}$, selecting the one with the highest reward \cite{wuInferenceScalingLaws2024,snellScalingLLMTestTime2024a,liuCan1BLLM2025}. The key distinction is that internal TTS fine-tunes $\pi_\theta$ to generate a long CoT reasoning path, while external TTS doesn't require fine-tuning.
There are various external TTS methods \cite{lightmanLetsVerifyStep2023, yaoTreeThoughtsDeliberate2023, snellScalingLLMTestTime2024a, wangSelfConsistencyImprovesChain2022}, and we briefly introduce two representative ones. The first is \textit{Best-of-N} sampling. Given a question \(q\), \(N\) independent and complete reasoning paths \(\{h_i\}_{i=1}^N\) are sampled from \(\pi_\theta(\cdot\mid q)\). Each candidate path \(h_i\) is evaluated by the PRM \(R_\phi:\mathcal{Q}\times \mathcal{H}\to[0,1]\), and outputs a reward \(r_i\in[0,1]\). The path with the highest reward is selected for the final output. The second is \textit{Beam Search}. Starting from the initial input \(q\), \(N\) candidate first steps \(\{z_{1,i}\}_{i=1}^N\) are sampled from \(\pi_\theta(\cdot\mid q)\). The PRM scores each first step, and the top \(N/M\) highest-scoring steps are retained, where \(M\) is the beam width. For each retained step, \(M\) next steps are sampled to expand into a total of \(N\) second-step paths. This procedure is repeated iteratively: at each step, paths are expanded, scored, filtered, and expanded again, until \(N\) complete reasoning paths are produced. Finally, the PRM evaluates the complete paths, and the highest-scoring path is selected.

A key challenge of external TTS is how to scale compute optimally. That is, given a fixed compute budget $C$, how to select the optimal hyperparameter $\psi$ that maximizes the probability of producing the correct answer for a given problem $q$. We follow \cite{snellScalingLLMTestTime2024a} and formalize the objective as:
\begin{equation}
\label{eq:compute_optimal}
    \psi^*_{q,a^*(q)}(C) = \mathop{\arg\max}_\psi(\mathbb{E}_{a\sim \text{Target}(\psi,C,q)}\big[\mathbbm{1}_{(a=a^*(q))}\big]),
\end{equation}
where $\mathbbm{1}_{(a=a^*(q))}$ is the indicator function that equals $1$ if the selected answer $a$ matches the ground-truth answer $a^*(q)$, and $0$ otherwise. $\text{Target}(\psi,C,q)$ denotes the distribution over outputs induced by executing the reasoning process under hyperparameter $\psi$ and budget $C$ on question $q$, and $\psi^*_{q,a^*(q)}(C)$ represents the optimal TTS strategy. The hyperparameter configuration $\psi$ includes, but is not limited to: \textbf{(i)} the choice of TTS strategy, such as Best-of-N, Beam Search, or other chain-of-thought-based methods; \textbf{(ii)} sampling parameters used during generation, such as top-$k$ truncation, top-$p$ truncation, and temperature; and \textbf{(iii)} for Beam Search, the beam width $M$ maintained at each reasoning step. The compute budget $C$ can be interpreted in multiple ways depending on the context. It can refer to the maximum number of tokens generated during test-time inference, or, as defined in \cite{snellScalingLLMTestTime2024a}, the total number of reasoning paths.  

\subsection{PRM: Process Reward Model}
\label{sec:prm}

The PRM \(R_\phi\) plays a central role in external TTS. It outputs a score for each reasoning step $z_t$ by inputting the reasoning prefix \(h_{t} = (z_{1}, \dots, z_{t})\) and the input question \(q\). A PRM is typically implemented by appending a linear prediction head to another LLM (different from $\pi_\theta$) and then fine-tuning the entire network on supervised training data \cite{uesato2022solving, lightmanLetsVerifyStep2023, zhangLessonsDevelopingProcess2025,wangMathShepherdVerifyReinforce2024}. The dataset $\mathcal{D} = \{(q_i, \{(h_{i,t}, y_{h_{i,t}})\}_{t=1}^{T_i})\}_{i=1}^{n}$ for training PRM consists of $n$ questions and corresponding reasoning steps $\{h_{i,t}\}_{t=1}^{T_i}$, where \(q_i\) denotes the \(i\)-th input question, \(h_{i,t} = (z_{i,1}, \dots, z_{i,t})\) represents the reasoning prefix up to step \(t\) for question \(q_i\) and \(y_{h_{i,t}} \in \{0,1\}\) is the quality label for step \(h_{i,t}\), with 1 indicating ``good" and 0 indicating ``bad". The label collection process can be found in \cite{lightmanLetsVerifyStep2023, wangMathShepherdVerifyReinforce2024}. Given the above training data, the PRM training objective for each question $q_i$ is formulated as:
\begin{equation}
    \label{eq:prm}
\mathcal{L}_{\text{PRM}} = \sum_{t=1}^{T_i} \Big( y_{h_{i,t}} \log r_{h_{i,t}} + (1 - y_{h_{i,t}}) \log (1 - r_{h_{i,t}}) \Big),
\end{equation}
where \(r_{h_{i,t}} = R_\phi(q_i, h_{i,t})\). After training, the parameters of PRM are frozen. Within the external TTS framework, the trained PRM is expected to assign reliable scores to the novel reasoning steps produced by the frozen policy $\pi_\theta$ when confronted with unseen questions.

\subsection{Motivation: The Relationship between TCO and PRM}
\label{sec:motivation}
Intuitively, the score of reasoning steps produced by the PRM directly influences the selection of reasoning paths, thereby can affect both the final answer accuracy and the computational consumption during inference. 
Since a PRM is usually obtained by fine-tuning an LLM on a limited set of training data, its prediction may generalize poorly to unseen questions. Therefore, investigating how such generalization ability of PRM affect TCO is critical for designing more efficient TTS methods.
Specifically, we aim to address the following three key questions: \textbf{(i)} Under a fixed compute budget and a fixed reasoning strategy, how does the generalization ability of the PRM affect the accuracy of the final answer? \textbf{(ii)} Given the accuracy of the target answer, how does the generalization ability of the PRM influence the required compute budget? \textbf{(iii)} How can we dynamically allocate compute during inference based on the reward model's generalization behavior, to improve the final answer accuracy under a fixed total compute budget?
The first two questions aim to characterize how the generalization ability of the PRM affects TCO, while the third question focuses on designing an external TTS method that achieves TCO by leveraging the theoretical insights.

\section{Theoretical Analysis}
\label{sec:theory}
In this section, we develop a theoretical framework to answer the three questions above. We begin by modeling an upper bound on the generalization error of PRM within the PAC-Bayes framework. Next, to address the first question, we analyze how this bound affects answer accuracy through the risk of mis-ranking candidate paths. For the second question, we derive how this bound affects the compute budget required for a desired accuracy level. Finally, to address the third question, we examine how the theoretical findings inform the design of external TTS methods.

\subsection{Generalization Bounds for Reward Models}
\label{sec:bound}

As discussed in Section~\ref{sec:prm}, the PRM is typically trained using supervised learning on a limited set of annotated examples. However, at test time, the PRM must evaluate inputs that include not only previously unseen questions but also new reasoning paths generated by the policy model. In such cases, the PRM is still expected to assign reliable scores to candidate reasoning paths. We refer to this ability as the generalization ability of PRM, which we identify as a key factor influencing both reasoning performance and computational efficiency.

Let $\mathcal{D}$ be an unknown data distribution over $\mathcal{Q} \times \mathcal{H} \times \mathcal{Y}$, where each data point consists of a question $q$, a reasoning path $h$, and a binary label $y$ indicating whether the path is helpful for solving the question. Let $\phi \in \Phi$ denote the parameters of PRM $R_\phi$, $\Phi$ is the parameter space. Let $\ell(R_\phi(q,h), y) =|R_\phi(q,h)-y|\in[0,1]$ be the absolute error between the model's output to a ground-truth label $y \in \{0,1\}$. Under the following assumption:
\begin{assumption}
\label{ass:indist}
The data sample $(q,h,y)$ are drawn i.i.d. from the distribution \( \mathcal{D} \).
\end{assumption}
Let $\mathcal{L}_{\mathcal{D}}(\phi) = \mathbb{E}_{(q,h,y) \sim \mathcal{D}} \left[\ell(R_\phi(q,h), y)\right]$ be the population risk, and $\mathcal{L}_{\mathcal{S}}(\phi) = \frac{1}{n} \sum_{i=1}^n \ell(R_\phi(q_i,h_i), y_i)$ be the empirical risk on a training dataset $\mathcal{S} = \{(q_i, h_i, y_i)\}_{i=1}^{n}$. The generalization error $\varepsilon_{\text{gen}}(\phi)$ is defined as:
\begin{equation}
    \varepsilon_{\text{gen}}(\phi) := \mathcal{L}_{\mathcal{D}}(\phi) - \mathcal{L}_{\mathcal{S}}(\phi)
\end{equation}
It quantifies how much the model's performance on unseen data deviates from its performance on the training set $\mathcal{S}$. In practice, a PRM is typically trained to minimize the empirical loss on $\mathcal{S}$ \cite{lightmanLetsVerifyStep2023,wangMathShepherdVerifyReinforce2024}, resulting in a small training loss. Thus, the generalization error captures the extent to which the predicted rewards deviate from the ground-truth rewards on unseen reasoning paths and questions.

We adopt the PAC-Bayes framework to analyze $\varepsilon_{\text{gen}}(\phi)$. Let $(\Phi, \mathcal{B})$ be the measurable space of model parameters $\phi$, where $\mathcal{B}$ is the Borel $\sigma$-algebra on $\Phi$. Let $\mathcal{P}(\Phi)$ denote the set of all probability measures over $(\Phi, \mathcal{B})$. A prior distribution $P \in \mathcal{P}(\Phi)$ is a probability distribution over model parameters $\phi$, which is the learner’s initial assumption before seeing data. After seeing a training dataset $\mathcal{S} = \{(q_i, h_i, y_i)\}_{i=1}^n \sim \mathcal{D}^n$, the learner selects a posterior distribution $Q \in \mathcal{P}(\Phi)$ over parameters $\phi$, which is the learner’s belief after observing the training set. 
\begin{theorem}[PAC-Bayes Generalization Bound for PRMs]
\label{thm:pac-bayes}
Let $P, Q \in \mathcal{P}(\Phi)$ be any prior and posterior distributions over $\phi$, and let $\ell$ be a bounded loss function taking values in $[0,1]$. Then, for any $\delta \in (0,1]$, with probability at least $1 - \delta$ over the choice of training set $\mathcal{S} \sim \mathcal{D}^n$, the following inequality holds:
\begin{equation}
\label{eq:upper_bound}
    \mathbb{E}_{\phi \sim Q} \left[ \mathcal{L}_{\mathcal{D}}(\phi) \right]
    \;\leq\;
    \mathbb{E}_{\phi \sim Q} \left[ \mathcal{L}_{\mathcal{S}}(\phi) \right]
    +
    \sqrt{\frac{ \mathrm{KL}(Q \| P) + \log\frac{n}{\delta} }{2(n - 1)}}.
\end{equation}
\end{theorem}
The proof is provided in Appendix~\ref{app:proof_1}. Equation \ref{eq:upper_bound} can be rewritten as the expected generalization error:
\begin{equation}
\label{eq:upper_bound_2}
    \mathbb{E}_{\phi\sim Q}\!\bigl[\varepsilon_{\text{gen}}(\phi)\bigr]
\le
\sqrt{\,
        \frac{\mathrm{KL}(Q\Vert P)\;+\;\log\frac{n}{\delta}}
             {2\,(n-1)}
      }.
\end{equation}
Equation \ref{eq:upper_bound_2} shows that, with probability at least $1-\delta$ over the draw of the training set $\mathcal{S}$, the expected deviation of the predicted reward from the true reward under the learner’s belief $Q$ is upper-bounded by the sample size $n$ and the divergence $\mathrm{KL}(Q\Vert P)$ between posterior and prior. In practice, the PRM is typically fixed after training, which corresponds to using a Dirac posterior $Q = \delta_{\hat{\phi}}$. In this case, the PAC-Bayes bound reduces to a pointwise guarantee $\varepsilon_{\text{gen}}(\hat\phi)\leq\sqrt{
(\log (1/P(\hat{\phi})) + \log (n/\delta))/{2(n - 1)}}$, and the KL term becomes $\log(1/P(\hat{\phi}))$, reflecting how well the learned parameters align with the prior. Next, we analyze how this bound influences the final answer accuracy and how it relates to the problem of TCO.

\subsection{Impact of Reward Model Generalization on Answer Accuracy}
\label{sec:accuracy}

In the external TTS framework, the PRM selects the path with the highest predicted score, and the answer is correct only if this selected path is also the truly highest‑reward path. When the output of PRM is accurate, the top‑scored path indeed has the highest true reward. When the predicted scores deviate from the true rewards, two cases arise: \textbf{(i)} the predicted top path still coincides with the true top path; or \textbf{(ii)} prediction error causes the true best path to be ranked lower and hence not selected. The second case may lead the system to choose a suboptimal path and thereby reduce answer accuracy. To quantify this effect, we develop a theoretical framework that relates the generalization error of PRM to the accuracy of the selected answer.

Let $\mathcal{H} = \{h_1, h_2, \dots, h_N\}$ denote the set of candidate reasoning paths independently sampled from the policy model $\pi_\theta(\cdot \mid q)$ for a given question $q$, i.e., $\mathcal{H} \sim \pi_\theta^{\otimes N}$. Here, $\pi_\theta^{\otimes N}$ denotes the joint distribution of $N$ independent samples from $\pi_\theta(\cdot|q)$. The goal of external TTS is to select the highest-scoring path $h_{\text{sel}}$ according to a learned PRM $R_\phi(q, h)$, i.e., $h_{\text{sel}} = \arg\max_{h \in \mathcal{H}} R_\phi(q, h)$, with the hope that the selected path yields the correct answer, i.e., $a(h_{\text{sel}}) = a^*(q)$. Therefore, assuming access to a ground-truth reward function $R^*(q, h) \in [0, 1]$, it is reasonable to assume that a reasoning path with a sufficiently high ground-truth reward should lead to a correct answer. 
\begin{assumption}[Path-to-Answer Correctness]
\label{assump:path-to-answer}
There exists a threshold $\tau \in (0,1]$ such that for any $h \in \mathcal{H}$, if $R^*(q, h) \ge \tau$, then $a(h,q) = a^*(q)$, where $a(h)$ denotes the final answer of path $h$.
\end{assumption}
Furthermore, motivated by empirical observations \cite{brownLargeLanguageMonkeys2024}, we assume that as the number of sampled paths increases, the probability that $\mathcal{H}$ contains at least one high-reward path approaches 1. Formally:
\begin{assumption}[Asymptotic Coverage]
\label{assump:coverage}
Let $p_{N,\tau}(q) := \Pr_{\mathcal{H} \sim \pi_\theta^{\otimes N}} \big[ \exists h \in \mathcal{H},\; R^*(q, h) \ge \tau \big]$. Then $\lim_{N \to \infty} p_N(q) = 1$.
\end{assumption}
Under these assumptions, the only remaining source of error lies in the ranking behavior of the PRM: even if a high-quality path is present in the candidate set $\mathcal{H}$, the PRM may fail to rank it highest due to the generalization error, which we can upper-bounded according to Equation~\ref{eq:upper_bound_2}. Define this upper-bound as $\varepsilon$, we propose the following theorem.  
\begin{theorem}[Answer–Accuracy Bound with Reward‐Gap]
\label{thm:accuracy}
Let \(q\) be a fixed question, and let $\mathcal{H} = \{h_{1},\dots,h_{N}\} \sim\pi_{\theta}^{\otimes N}$ be \(N\) i.i.d.\ candidate reasoning paths.  Define $h^{*} \;=\;\arg\max_{h\in\mathcal{H}}R^{*}(q,h)$, and $h_{\mathrm{sel}} =\arg\max_{h\in\mathcal{H}}R_{\phi}(q,h)$. The reward‐gap $\gamma(q)=R^{*}\bigl(q,h^{*}\bigr)-\max_{h\in\mathcal{H}\setminus\{h^{*}\}}
       R^{*}(q,h)\ge 0.$ Let
$p_{N,\tau}(q)=\Pr_{\mathcal{H}\sim\pi_{\theta}^{\otimes N}}\bigl[\exists\,h\in\mathcal{H}:\;R^{*}(q,h)\ge\tau\bigr].$
Suppose further that the following hold:
\begin{itemize}
    \item There exist \(\varepsilon\in(0,1]\) and \(\delta\in(0,1)\) such that $\Pr\!\Bigl[\sup_{h\in\mathcal{H}}\bigl|R_{\phi}(q,h)-R^{*}(q,h)\bigr|\le\varepsilon\Bigr]\;\ge\;1-\delta$. And denote the high‐probability event by \(\mathcal{G}=\{\sup_{h}|R_{\phi}-R^{*}|\le\varepsilon\}\).
    \item Conditioned on \(\mathcal{H}\), the deviations \(\Delta_{h}=R_{\phi}(q,h)-R^{*}(q,h)\) are independent, mean‐zero, and satisfy \(|\Delta_{h}|\le\varepsilon\) almost surely.
    \item Assumptions \ref{assump:path-to-answer} and Assumption \ref{assump:coverage} hold.
\end{itemize}
Then the probability of selecting a correct answer satisfies
\begin{equation}
\label{eq:accuracy‐gap‐bound}
  \Pr\bigl[a(h_{\mathrm{sel}})=a^{*}(q)\bigr]
  \;\ge\;
  p_{N,\tau}(q)\,\Bigl[
    1 \;-\;\delta
        \;-\;(N-1)\,\exp\!\bigl(-\frac{\gamma(q)^{2}}{8\varepsilon^{2}}\bigr)
  \Bigr].
\end{equation}
\end{theorem}
The proof is provided in Appendix~\ref{app:proof_2}. Theorem \ref{thm:accuracy} shows that the answer accuracy of external TTS is lower-bounded by $p_{N}(q)$ and $\exp(-\gamma(q)^{2}/(8\varepsilon^{2}))$, where $p_{N}(q)$ reflects the chance of sampling at least one high-quality path under fixed budget $N$, and $\varepsilon$ reflect the generalization error of PRM. This implies that the answer accuracy depends jointly on the sampling ability of the policy model and the generalization error of the PRM.

\subsection{Impact of Reward Model Generalization on Compute Budget}
\label{sec:compute}

Next, based on Theorem~\ref{thm:accuracy}, we propose the following corollary to describe the budget requirement as a function of the upper-bound of the generalization error of PRM $\varepsilon$ and reward gap $\gamma(q)$.
\begin{corollary}[Target Accuracy Constraint on Sampling and Margin]
\label{cor:budget}
Given a generalization error bound of PRM \(\varepsilon > 0\), a confidence parameter \(\delta\in(0,1)\), and a target answer accuracy level \(\alpha \in (0,1)\). Under the assumptions of Theorem~\ref{thm:accuracy}, if one wishes to guarantee $\Pr[a(h_{\mathrm{sel}})=a^*(q)] \;\ge\; \alpha$, then the sampling coverage probability must satisfy
\begin{equation}
\label{eq:alpha-constraint}
   p_{N,\tau}(q)
   \;\ge\;
   \frac{\alpha}
        {1 - \delta - (N\!-\!1)\,
            \exp\!\left(-\gamma(q)^2\,/\,8\varepsilon^2\right)} .
\end{equation}
\end{corollary}
This corollary shows that, for higher generalization error of PRM, more reasoning paths need to be sampled to guarantee a higher accuracy $\alpha$. Therefore, the generalization ability of the PRM also significantly affects the compute budget, which is used to achieve a higher accuracy.

\subsection{Inspiration for Designing External TTS Methods}
\label{sec:inspire}
The analysis in Section~\ref{sec:accuracy} and Section~\ref{sec:compute} addresses the first two problems we propose. For the third problem, Theorem~\ref{thm:accuracy} and Corollary~\ref{cor:budget} provide direct insight into the core objective of TCO as in Equation \ref{eq:compute_optimal}: selecting search hyperparameters that maximize answer accuracy under a fixed compute budget. 
Specifically, Equation~\ref{eq:accuracy‐gap‐bound} shows that answer accuracy increases with larger reward margin $\gamma(q)$, while it is negatively impacted by the generalization error bound $\varepsilon$ of the PRM. Although $\varepsilon$ is unknown at test time, the reward gap $\gamma(q)$ can be influenced by the search configuration. For example, adjusting sampling parameters such as top-$k$, top-$p$, and temperature can affect the diversity of candidate reasoning paths generated by the policy model. This, in turn, modifies the reward separation among candidates. 
However, several challenges arise in practice: \textbf{(i)} different PRMs may have different generalization behaviors, requiring different reward margins \(\gamma(q)\) to ensure reliable selection; \textbf{(ii)} generating candidate paths that satisfy a desired reward margin may require multiple sampling rounds. 
These observations motivate the design of a dynamic control mechanism that makes decisions based on generalization behaviors of PRMs.

\section{Methodology}
As we discussed in Section~\ref{sec:inspire}, an effective inference-time strategy should \textbf{(i)} be aware of the generalization behavior of the reward model, and \textbf{(ii)} increase $\gamma(q)$ dynamically without inducing additional compute cost. To address the first problem, we propose using the structural sparsity of $\phi$ as a proxy for estimating its generalization capacity. To address the second problem, we propose \emph{Compute-Aware Tree Search} (CATS), a dynamic compute allocation framework based on A2C framework. During inference, the actor observes the current reasoning state and outputs search hyperparameters. The critic estimates the utility of each action and provides feedback to optimize the actor via reinforcement learning. This design allows CATS to adaptively adjust computation at each reasoning step while maintaining a global compute budget.


\subsection{Estimation of $\varepsilon$ via Sparsity}

In practice, the true generalization error $\varepsilon_{\text{gen}}(\hat{\phi})$ in Theorem \ref{thm:pac-bayes} is unobservable and its PAC-Bayes upper bound depends on the prior density $P(\hat{\phi})$, which is rarely known in closed form. However, under structural assumptions on $\hat{\phi}$, we can approximate $\log(1/P(\hat{\phi}))$ using model-dependent statistics as a proxy. One common and well-motivated assumption is parameter sparsity, which reflects the idea that only a small subset of model parameters are relevant for capturing the reward signal. Sparsity-based priors have been widely used in PAC-Bayesian analysis to obtain non-vacuous generalization bounds \cite{muthukumar2023sparsity, lotfi2022pac}, and have also proven effective in various practical settings \cite{roy2021efficient, jiang2024minference}. Under sparsity-based priors, the KL-divergence term can be upper-bounded by a function of the number of nonzero parameters in $\hat{\phi}$. This yields the following sparsity-induced bound on the generalization error:
\begin{equation}
    \varepsilon_{\text{gen}}(\hat{\phi})
\;\le\;
\sqrt{
\frac{c \cdot \|\hat{\phi}\|_0 \cdot \log d + \log \frac{n}{\delta}}{2(n - 1)}
}.
\end{equation}
This expression provides a practical surrogate: models with fewer active parameters are expected to generalize better.
We provide empirical evidence in Appendix~\ref{app:sparsity}.

\subsection{Compute-Aware Tree Search}

To dynamically allocate compute budget at each reasoning step based on the generalization behavior of the reward model, we propose Compute-Aware Tree Search (CATS). In this approach, we formalize the reasoning process as a Markov Decision Process (MDP). Formally, we define the reasoning control problem as  $(\mathcal{S}, \mathcal{A}, P, r, \gamma)$. The state space $\mathcal{S}$ captures the current search context, including: the number of candidate paths at the current step, their associated reward scores, parameter sparsity of the reward model, and the maximum candidate paths that can be sampled. The action space $\mathcal{A}$ consists of a set of search hyperparameter configurations, including the number of additional candidates to sample, the number of candidates to retain for the next step, and the sampling parameters (e.g., top-$p$, top-$k$, and temperature). The transition function $P$ is deterministic: the next state is determined by applying the chosen action, either by sampling additional candidates and then retaining a subset, or by directly retaining a subset of existing candidates to advance to the next reasoning step. The reward function $r(s_t, a_t)$ is defined as follows:
\begin{equation}
    r(s_t,a_t) = -\lambda_c\cdot C(a_t) + \lambda_m\cdot \Delta_m(s_t,a_t) + \lambda_r\cdot\max_{h\in\mathcal{H}}R_\phi(q,h),
\end{equation}
where $C(a_t)$ denotes the additional candidate path incurred by action $a_t$, $\Delta_m(s_t,a_t)$ denotes the reward gap between the retained paths and discarded paths, $\max_{h\in\mathcal{H}}R_\phi(q,h)$ is the highest score of the candidate paths, $\lambda_c,\lambda_m,\lambda_r$ are hyperparameters. These rewards encourage high-quality generations and mitigate the risk of mis-pruning good paths. And $\gamma \in (0, 1]$ is a discount factor. Under this formulation, the objective of CATS is to learn a control policy $\pi_\nu(a_t \mid s_t)$ that maximizes the expected return throughout the reasoning process.

We employ an A2C framework \cite{sutton1999policy} to optimize the tree expansion policy via single-step temporal difference (TD) learning. The actor network $\pi_\nu(a_t \mid s_t)$ is based on a multi-layer perceptron (MLP) that outputs action probability. The critic network $V_\xi(s_t)$ is implemented as a separate MLP that predicts the scalar value of a given state. At each search step, the agent collects transition tuples $(s_t, a_t, r_t, s_{t+1})$ and computes the TD error:
\begin{equation}
    \delta_t = r_t + \gamma V_\xi(s_{t+1}) - V_\xi(s_t),
\end{equation}
which serves both as a regression target for the critic and as an advantage estimate for the actor. The critic is trained to minimize the squared TD error, while the actor is trained to maximize the expected return using the advantage-weighted log-probability objective:
\begin{equation}
\label{eq:obj_actor_critic}
    \mathcal{L}_{\text{critic}}(\xi) = \frac{1}{2} \left( \delta_t \right)^2,\mathcal{L}_{\text{actor}}(\nu) = -\log \pi_\nu(a_t \mid s_t) \cdot \delta_t.
\end{equation}
Gradients are computed with respect to $\xi$ and $\nu$, and updates are applied after each environment step. 
By optimizing Equation~\ref{eq:obj_actor_critic}, the actor learns to produce actions at each step that maximize reward. During the testing phase, the actor can generate candidate reasoning paths with higher PRM scores and larger reward gaps, which helps prevent mis-ranking and ultimately improves answer accuracy. The pseudo codes for training and using CATS are provided in Appendix~\ref{app:pse}.

\section{Experiments}
\label{sec:experiments}
\subsection{Implementation Details}
To train the actor and critic networks, we follow the procedure in \cite{lightmanLetsVerifyStep2023} and construct a training set using 12,000 examples from the MATH dataset \cite{hendrycksmath2021}. Each training sample consists of a question $q$ and its corresponding ground-truth answer $a^*(q)$. During training, we fix a policy LLM to generate candidate answers and collect data by scoring the reasoning paths under different PRMs. These trajectories are then used to train both the Actor and Critic networks. The Actor network is implemented as a two-layer MLP with a hidden dimension of $128$ while the Critic network is also implemented as a two-layer MLP with a hidden dimension of $256$. The hyperparameters $\lambda_c=0.2,\lambda_m=0.5,\lambda_r=0.3$. We use the Adam optimizer with a learning rate of $1\times10^{-3}$ and train the models under different compute budgets. Experiments are conducted on 8 A800 GPUs. The ablation study is in Appendix~\ref{app:abl}.

\begin{figure}[t]
    \centering
    \includegraphics[width=.85\linewidth]{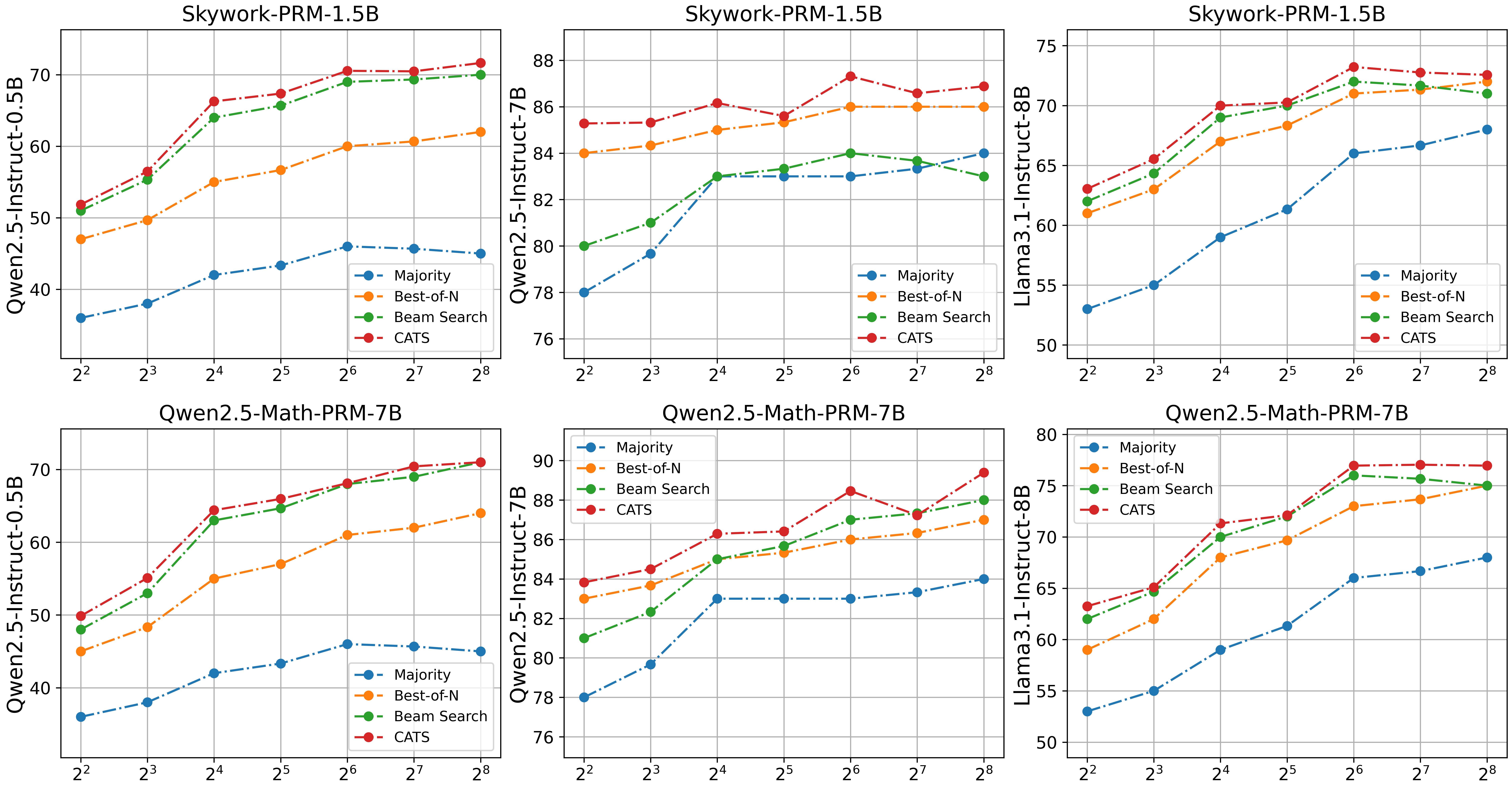}
    \caption{The comparison results on the MATH-500 dataset for different policy models and PRMs.}
    \label{fig:math}
\end{figure}

\begin{figure}[t]
    \centering
    \includegraphics[width=.85\linewidth]{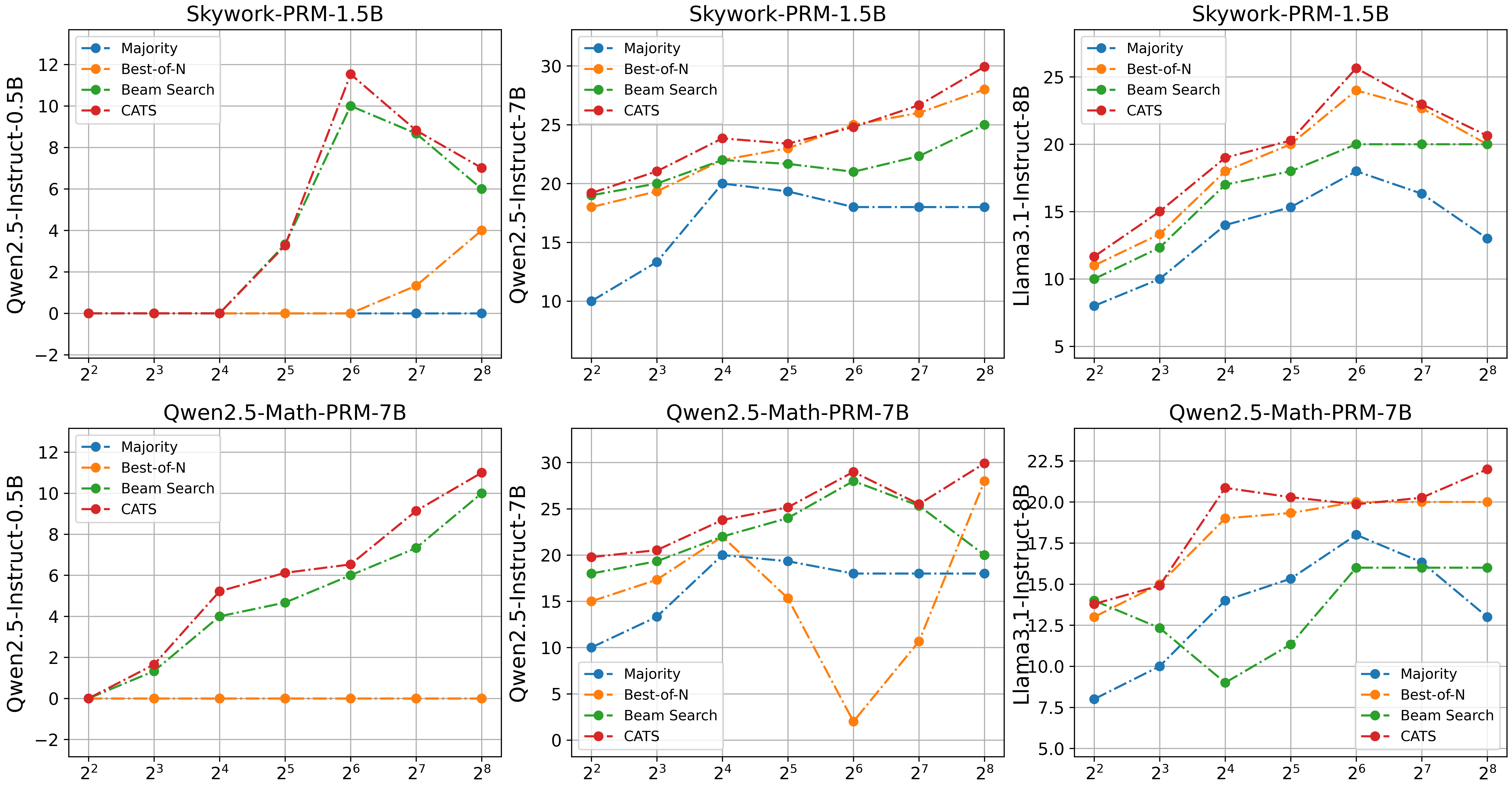}
    \caption{The comparison results on the AIME24 dataset for different policy models and PRMs.}
    \label{fig:aime}
\end{figure}

\subsection{Experimental Setup}
We evaluate the proposed CATS method on two mathematical reasoning benchmarks: MATH-500 \cite{lightmanLetsVerifyStep2023} and AIME24 \cite{ai-mo2024}. To assess the generality of our approach, we evaluate CATS across a diverse set of frozen policy models, including LLaMA3.1-Instruct-8B, LLaMA3.2-Instruct-1B, Qwen2.5-Instruct (0.5B, 3B, and 7B). For the reward models, we include Math-Shepherd-PRM-7B, RLHFlow-PRM-Mistral-8B, RLHFlow-PRM-DeepSeek-8B, Skywork-PRM-1.5B, and Qwen2.5-Math-PRM-7B. The maximum number of candidate reasoning paths is 256. We compare CATS with three external TTS methods: Best-of-N, Beam Search with width $M=4$, and Majority Voting. For each method, we report answer accuracy as a function of the number of candidate paths, using $N \in \{4, 8, 16, 32, 64, 128, 256\}$. The full results are provided in Appendix~\ref{app:full}. And we also present the comparison results with other external TTS methods in Appendix~\ref{sec:compare}.

\subsection{Results}
The performance of Qwen2.5-Instruct-0.5B, Qwen2.5-Instruct-7B, and LLaMA3.1-Instruct-8B on the MATH-500 dataset, evaluated under two PRMs: Skywork-PRM-1.5B and Qwen2.5-Math-PRM-7B, is shown in Figure~\ref{fig:math}. From the results, we observe that for different compute budgets and PRMs, the best-performing baseline (excluding CATS) varies. In contrast, CATS consistently outperforms all baselines across all budget levels and PRMs. The results on the AIME24 dataset are shown in Figure~\ref{fig:aime}. Although AIME24 poses greater challenges than MATH-500, CATS continues to outperform other external TTS methods. These findings confirm the effectiveness of our proposed approach.

\section{Conclusion}
In this work, we analyze how the generalization error of the PRM influences the performance of the external TTS. By quantifying the mis-ranking risk induced by reward prediction error, we derive an explicit lower bound involving the reward gap and path coverage probability, which motivates the need for adaptive control over reasoning computation. Building on this insight, we propose CATS, a dynamic inference strategy based on the A2C framework. CATS learns to allocate compute based on PRM proxies and effectively balances reward separation and candidate diversity. Extensive experiments on MATH and AIME24 demonstrate that CATS consistently outperforms standard search strategies across a wide range of policy models and PRMs.

\bibliography{main}

\begin{thebibliography}{10}

\bibitem{weiChainofthoughtPromptingElicits2022}
Jason Wei, Xuezhi Wang, Dale Schuurmans, Maarten Bosma, Fei Xia, Ed~Chi, Quoc~V. Le, and Denny Zhou.
\newblock Chain-of-thought prompting elicits reasoning in large language models.
\newblock {\em Advances in neural information processing systems}, 35:24824--24837, 2022.

\bibitem{kojimaLargeLanguageModels2022}
Takeshi Kojima, Shixiang~Shane Gu, Machel Reid, Yutaka Matsuo, and Yusuke Iwasawa.
\newblock Large language models are zero-shot reasoners.
\newblock {\em Advances in neural information processing systems}, 35:22199--22213, 2022.

\bibitem{yaoTreeThoughtsDeliberate2023}
Shunyu Yao, Dian Yu, Jeffrey Zhao, Izhak Shafran, Thomas~L. Griffiths, Yuan Cao, and Karthik~R. Narasimhan.
\newblock Tree of {{Thoughts}}: {{Deliberate Problem Solving}} with {{Large Language Models}}.
\newblock In {\em Thirty-Seventh {{Conference}} on {{Neural Information Processing Systems}}}, November 2023.

\bibitem{openai2024learning}
OpenAI.
\newblock Learning to reason with llms.
\newblock \url{https://openai.com/index/learning-to-reason-with-llms/}, 2024.
\newblock Accessed: 2025-04-29.

\bibitem{deepseek-aiDeepSeekR1IncentivizingReasoning2025}
{DeepSeek-AI}.
\newblock {{DeepSeek-R1}}: {{Incentivizing Reasoning Capability}} in {{LLMs}} via {{Reinforcement Learning}}, January 2025.

\bibitem{teamKimiK15Scaling2025}
Kimi Team, Angang Du, Bofei Gao, Bowei Xing, Changjiu Jiang, Cheng Chen, Cheng Li, Chenjun Xiao, Chenzhuang Du, Chonghua Liao, Chuning Tang, Congcong Wang, Dehao Zhang, Enming Yuan, Enzhe Lu, Fengxiang Tang, Flood Sung, Guangda Wei, Guokun Lai, Haiqing Guo, Han Zhu, Hao Ding, Hao Hu, Hao Yang, Hao Zhang, Haotian Yao, Haotian Zhao, Haoyu Lu, Haoze Li, Haozhen Yu, Hongcheng Gao, Huabin Zheng, Huan Yuan, Jia Chen, Jianhang Guo, Jianlin Su, Jianzhou Wang, Jie Zhao, Jin Zhang, Jingyuan Liu, Junjie Yan, Junyan Wu, Lidong Shi, Ling Ye, Longhui Yu, Mengnan Dong, Neo Zhang, Ningchen Ma, Qiwei Pan, Qucheng Gong, Shaowei Liu, Shengling Ma, Shupeng Wei, Sihan Cao, Siying Huang, Tao Jiang, Weihao Gao, Weimin Xiong, Weiran He, Weixiao Huang, Wenhao Wu, Wenyang He, Xianghui Wei, Xianqing Jia, Xingzhe Wu, Xinran Xu, Xinxing Zu, Xinyu Zhou, Xuehai Pan, Y.~Charles, Yang Li, Yangyang Hu, Yangyang Liu, Yanru Chen, Yejie Wang, Yibo Liu, Yidao Qin, Yifeng Liu, Ying Yang, Yiping Bao, Yulun Du, Yuxin Wu, Yuzhi Wang, Zaida Zhou,
  Zhaoji Wang, Zhaowei Li, Zhen Zhu, Zheng Zhang, Zhexu Wang, Zhilin Yang, Zhiqi Huang, Zihao Huang, Ziyao Xu, and Zonghan Yang.
\newblock Kimi k1.5: {{Scaling Reinforcement Learning}} with {{LLMs}}, January 2025.

\bibitem{snellScalingLLMTestTime2024a}
Charlie~Victor Snell, Jaehoon Lee, Kelvin Xu, and Aviral Kumar.
\newblock Scaling {{LLM Test-Time Compute Optimally Can}} be {{More Effective}} than {{Scaling Parameters}} for {{Reasoning}}.
\newblock In {\em The {{Thirteenth International Conference}} on {{Learning Representations}}}, October 2024.

\bibitem{liuCan1BLLM2025}
Runze Liu, Junqi Gao, Jian Zhao, Kaiyan Zhang, Xiu Li, Biqing Qi, Wanli Ouyang, and Bowen Zhou.
\newblock Can {{1B LLM Surpass 405B LLM}}? {{Rethinking Compute-Optimal Test-Time Scaling}}, February 2025.

\bibitem{wuInferenceScalingLaws2024}
Yangzhen Wu, Zhiqing Sun, Shanda Li, Sean Welleck, and Yiming Yang.
\newblock Inference {{Scaling Laws}}: {{An Empirical Analysis}} of {{Compute-Optimal Inference}} for {{LLM Problem-Solving}}.
\newblock In {\em The {{Thirteenth International Conference}} on {{Learning Representations}}}, October 2024.

\bibitem{bestaGraphThoughtsSolving2024}
Maciej Besta, Nils Blach, Ales Kubicek, Robert Gerstenberger, Michal Podstawski, Lukas Gianinazzi, Joanna Gajda, Tomasz Lehmann, Hubert Niewiadomski, and Piotr Nyczyk.
\newblock Graph of thoughts: {{Solving}} elaborate problems with large language models.
\newblock In {\em Proceedings of the {{AAAI Conference}} on {{Artificial Intelligence}}}, volume~38, pages 17682--17690, 2024.

\bibitem{quOptimizingTestTimeCompute2025}
Yuxiao Qu, Matthew Y.~R. Yang, Amrith Setlur, Lewis Tunstall, Edward~Emanuel Beeching, Ruslan Salakhutdinov, and Aviral Kumar.
\newblock Optimizing {{Test-Time Compute}} via {{Meta Reinforcement Fine-Tuning}}, March 2025.

\bibitem{lightmanLetsVerifyStep2023}
Hunter Lightman, Vineet Kosaraju, Yura Burda, Harri Edwards, Bowen Baker, Teddy Lee, Jan Leike, John Schulman, Ilya Sutskever, and Karl Cobbe.
\newblock Let's {{Verify Step}} by {{Step}}, May 2023.

\bibitem{wangMathShepherdVerifyReinforce2024}
Peiyi Wang, Lei Li, Zhihong Shao, R.~X. Xu, Damai Dai, Yifei Li, Deli Chen, Y.~Wu, and Zhifang Sui.
\newblock Math-{{Shepherd}}: {{Verify}} and {{Reinforce LLMs Step-by-step}} without {{Human Annotations}}, February 2024.

\bibitem{hendrycksmath2021}
Dan Hendrycks, Collin Burns, Saurav Kadavath, Akul Arora, Steven Basart, Eric Tang, Dawn Song, and Jacob Steinhardt.
\newblock Measuring mathematical problem solving with the math dataset.
\newblock {\em NeurIPS}, 2021.

\bibitem{ai-mo2024}
{AI-MO}.
\newblock {AIME 2024}.
\newblock \url{https://huggingface.co/datasets/AI-MO/aimo-validation-aime}, July 2024.
\newblock Accessed 2024-07.

\bibitem{yang2024qwen2}
An~Yang, Baosong Yang, Beichen Zhang, Binyuan Hui, Bo~Zheng, Bowen Yu, Chengyuan Li, Dayiheng Liu, Fei Huang, Haoran Wei, et~al.
\newblock Qwen2. 5 technical report.
\newblock {\em arXiv preprint arXiv:2412.15115}, 2024.

\bibitem{grattafiori2024llama}
Aaron Grattafiori, Abhimanyu Dubey, Abhinav Jauhri, Abhinav Pandey, Abhishek Kadian, Ahmad Al-Dahle, Aiesha Letman, Akhil Mathur, Alan Schelten, Alex Vaughan, et~al.
\newblock The llama 3 herd of models.
\newblock {\em arXiv preprint arXiv:2407.21783}, 2024.

\bibitem{meta2024llama3}
{Meta AI}.
\newblock Llama 3 to connect 2024: Vision, edge, and mobile devices.
\newblock \url{https://ai.meta.com/blog/llama-3-2-connect-2024-vision-edge-mobile-devices/}, 2024.
\newblock Accessed: 2025-05-12.

\bibitem{zhangLessonsDevelopingProcess2025}
Zhenru Zhang, Chujie Zheng, Yangzhen Wu, Beichen Zhang, Runji Lin, Bowen Yu, Dayiheng Liu, Jingren Zhou, and Junyang Lin.
\newblock The {{Lessons}} of {{Developing Process Reward Models}} in {{Mathematical Reasoning}}, January 2025.

\bibitem{dong2024rlhf}
Hanze Dong, Wei Xiong, Bo~Pang, Haoxiang Wang, Han Zhao, Yingbo Zhou, Nan Jiang, Doyen Sahoo, Caiming Xiong, and Tong Zhang.
\newblock Rlhf workflow: From reward modeling to online rlhf.
\newblock {\em arXiv preprint arXiv:2405.07863}, 2024.

\bibitem{skywork2024}
{Skywork o1 Team}.
\newblock {Skywork-o1 open series}.
\newblock \url{https://huggingface.co/Skywork}, November 2024.
\newblock Accessed: 2025-05-12.

\bibitem{chenReasoningEraSurvey2025}
Qiguang Chen, Libo Qin, Jinhao Liu, Dengyun Peng, Jiannan Guan, Peng Wang, Mengkang Hu, Yuhang Zhou, Te~Gao, and Wanxiang Che.
\newblock Towards {{Reasoning Era}}: {{A Survey}} of {{Long Chain-of-Thought}} for {{Reasoning Large Language Models}}, April 2025.

\bibitem{madaanSelfRefineIterativeRefinement2023}
Aman Madaan, Niket Tandon, Prakhar Gupta, Skyler Hallinan, Luyu Gao, Sarah Wiegreffe, Uri Alon, Nouha Dziri, Shrimai Prabhumoye, Yiming Yang, Shashank Gupta, Bodhisattwa~Prasad Majumder, Katherine Hermann, Sean Welleck, Amir Yazdanbakhsh, and Peter Clark.
\newblock Self-{{Refine}}: {{Iterative Refinement}} with {{Self-Feedback}}.
\newblock In {\em Thirty-Seventh {{Conference}} on {{Neural Information Processing Systems}}}, November 2023.

\bibitem{saunders2022self}
William Saunders, Catherine Yeh, Jeff Wu, Steven Bills, Long Ouyang, Jonathan Ward, and Jan Leike.
\newblock Self-critiquing models for assisting human evaluators.
\newblock {\em arXiv preprint arXiv:2206.05802}, 2022.

\bibitem{shaoDeepSeekMathPushingLimits2024}
Zhihong Shao, Peiyi Wang, Qihao Zhu, Runxin Xu, Junxiao Song, Xiao Bi, Haowei Zhang, Mingchuan Zhang, Y.~K. Li, Y.~Wu, and Daya Guo.
\newblock {{DeepSeekMath}}: {{Pushing}} the {{Limits}} of {{Mathematical Reasoning}} in {{Open Language Models}}, April 2024.

\bibitem{selAlgorithmThoughtsEnhancing2024}
Bilgehan Sel, Ahmad Tawaha, Vanshaj Khattar, Ruoxi Jia, and Ming Jin.
\newblock Algorithm of {{Thoughts}}: {{Enhancing Exploration}} of {{Ideas}} in {{Large Language Models}}.
\newblock In {\em Forty-First {{International Conference}} on {{Machine Learning}}}, June 2024.

\bibitem{zhangAutomaticChainThought2022}
Zhuosheng Zhang, Aston Zhang, Mu~Li, and Alex Smola.
\newblock Automatic {{Chain}} of {{Thought Prompting}} in {{Large Language Models}}, October 2022.

\bibitem{brownLargeLanguageMonkeys2024}
Bradley Brown, Jordan Juravsky, Ryan Ehrlich, Ronald Clark, Quoc~V. Le, Christopher R{\'e}, and Azalia Mirhoseini.
\newblock Large {{Language Monkeys}}: {{Scaling Inference Compute}} with {{Repeated Sampling}}, December 2024.

\bibitem{uesato2022solving}
Jonathan Uesato, Nate Kushman, Ramana Kumar, Francis Song, Noah Siegel, Lisa Wang, Antonia Creswell, Geoffrey Irving, and Irina Higgins.
\newblock Solving math word problems with process-and outcome-based feedback.
\newblock {\em arXiv preprint arXiv:2211.14275}, 2022.

\bibitem{xieSelfEvaluationGuidedBeam2023}
Yuxi Xie, Kenji Kawaguchi, Yiran Zhao, Xu~Zhao, Min-Yen Kan, Junxian He, and Qizhe Xie.
\newblock Self-{{Evaluation Guided Beam Search}} for {{Reasoning}}, October 2023.

\bibitem{wangSelfConsistencyImprovesChain2022}
Xuezhi Wang, Jason Wei, Dale Schuurmans, Quoc~V. Le, Ed~H. Chi, Sharan Narang, Aakanksha Chowdhery, and Denny Zhou.
\newblock Self-{{Consistency Improves Chain}} of {{Thought Reasoning}} in {{Language Models}}.
\newblock In {\em The {{Eleventh International Conference}} on {{Learning Representations}}}, September 2022.

\bibitem{muthukumar2023sparsity}
Ramchandran Muthukumar and Jeremias Sulam.
\newblock Sparsity-aware generalization theory for deep neural networks.
\newblock In {\em The Thirty Sixth Annual Conference on Learning Theory}, pages 5311--5342. PMLR, 2023.

\bibitem{lotfi2022pac}
Sanae Lotfi, Marc Finzi, Sanyam Kapoor, Andres Potapczynski, Micah Goldblum, and Andrew~G Wilson.
\newblock Pac-bayes compression bounds so tight that they can explain generalization.
\newblock {\em Advances in Neural Information Processing Systems}, 35:31459--31473, 2022.

\bibitem{roy2021efficient}
Aurko Roy, Mohammad Saffar, Ashish Vaswani, and David Grangier.
\newblock Efficient content-based sparse attention with routing transformers.
\newblock {\em Transactions of the Association for Computational Linguistics}, 9:53--68, 2021.

\bibitem{jiang2024minference}
Huiqiang Jiang, Yucheng Li, Chengruidong Zhang, Qianhui Wu, Xufang Luo, Surin Ahn, Zhenhua Han, Amir~H Abdi, Dongsheng Li, Chin-Yew Lin, et~al.
\newblock Minference 1.0: Accelerating pre-filling for long-context llms via dynamic sparse attention.
\newblock {\em arXiv preprint arXiv:2407.02490}, 2024.

\bibitem{sutton1999policy}
Richard~S Sutton, David McAllester, Satinder Singh, and Yishay Mansour.
\newblock Policy gradient methods for reinforcement learning with function approximation.
\newblock {\em Advances in neural information processing systems}, 12, 1999.

\bibitem{mcallester1999pac}
David~A McAllester.
\newblock Pac-bayesian model averaging.
\newblock In {\em Proceedings of the twelfth annual conference on Computational learning theory}, pages 164--170, 1999.

\bibitem{seeger2002pac}
Matthias Seeger.
\newblock Pac-bayesian generalisation error bounds for gaussian process classification.
\newblock {\em Journal of machine learning research}, 3(Oct):233--269, 2002.

\bibitem{cobbe2021gsm8k}
Karl Cobbe, Vineet Kosaraju, Mohammad Bavarian, Mark Chen, Heewoo Jun, Lukasz Kaiser, Matthias Plappert, Jerry Tworek, Jacob Hilton, Reiichiro Nakano, Christopher Hesse, and John Schulman.
\newblock Training verifiers to solve math word problems.
\newblock {\em arXiv preprint arXiv:2110.14168}, 2021.

\bibitem{he2024olympiadbench}
Chaoqun He, Renjie Luo, Yuzhuo Bai, Shengding Hu, Zhen~Leng Thai, Junhao Shen, Jinyi Hu, Xu~Han, Yujie Huang, Yuxiang Zhang, Jie Liu, Lei Qi, Zhiyuan Liu, and Maosong Sun.
\newblock Olympiadbench: A challenging benchmark for promoting agi with olympiad-level bilingual multimodal scientific problems, 2024.

\bibitem{jiang2024technical}
Jinhao Jiang, Zhipeng Chen, Yingqian Min, Jie Chen, Xiaoxue Cheng, Jiapeng Wang, Yiru Tang, Haoxiang Sun, Jia Deng, Wayne~Xin Zhao, et~al.
\newblock Technical report: Enhancing llm reasoning with reward-guided tree search.
\newblock {\em arXiv preprint arXiv:2411.11694}, 2024.

\bibitem{wang2024litesearch}
Ante Wang, Linfeng Song, Ye~Tian, Baolin Peng, Dian Yu, Haitao Mi, Jinsong Su, and Dong Yu.
\newblock Litesearch: Efficacious tree search for llm.
\newblock {\em arXiv preprint arXiv:2407.00320}, 2024.

\bibitem{chen2024alphamath}
Guoxin Chen, Minpeng Liao, Chengxi Li, and Kai Fan.
\newblock Alphamath almost zero: process supervision without process.
\newblock {\em arXiv preprint arXiv:2405.03553}, 2024.

\bibitem{xie2024monte}
Yuxi Xie, Anirudh Goyal, Wenyue Zheng, Min-Yen Kan, Timothy~P Lillicrap, Kenji Kawaguchi, and Michael Shieh.
\newblock Monte carlo tree search boosts reasoning via iterative preference learning.
\newblock {\em arXiv preprint arXiv:2405.00451}, 2024.

\bibitem{li2024numinamath}
Jia Li, Edward Beeching, Lewis Tunstall, Ben Lipkin, Roman Soletskyi, Shengyi Huang, Kashif Rasul, Longhui Yu, Albert~Q Jiang, Ziju Shen, et~al.
\newblock Numinamath: The largest public dataset in ai4maths with 860k pairs of competition math problems and solutions.
\newblock {\em Hugging Face repository}, 13:9, 2024.

\bibitem{zhang2024llama}
Di~Zhang, Jianbo Wu, Jingdi Lei, Tong Che, Jiatong Li, Tong Xie, Xiaoshui Huang, Shufei Zhang, Marco Pavone, Yuqiang Li, et~al.
\newblock Llama-berry: Pairwise optimization for o1-like olympiad-level mathematical reasoning.
\newblock {\em arXiv preprint arXiv:2410.02884}, 2024.

\bibitem{zhang2024accessing}
Di~Zhang, Xiaoshui Huang, Dongzhan Zhou, Yuqiang Li, and Wanli Ouyang.
\newblock Accessing gpt-4 level mathematical olympiad solutions via monte carlo tree self-refine with llama-3 8b.
\newblock {\em arXiv preprint arXiv:2406.07394}, 2024.

\bibitem{zhang2025booststep}
Beichen Zhang, Yuhong Liu, Xiaoyi Dong, Yuhang Zang, Pan Zhang, Haodong Duan, Yuhang Cao, Dahua Lin, and Jiaqi Wang.
\newblock Booststep: Boosting mathematical capability of large language models via improved single-step reasoning.
\newblock {\em arXiv preprint arXiv:2501.03226}, 2025.

\end{thebibliography}
\bibliographystyle{unsrt}

\newpage
\appendix

\section*{Appendix}

\section{Limitations}
Our theoretical results are based on a set of reasonable assumptions, which may not fully hold in practical scenarios. However, these assumptions do not undermine the validity of the analysis, and our empirical results support the overall conclusions.

\section{Proof of Theorem~\ref{thm:pac-bayes}}
\label{app:proof_1}
\begin{theorem}[PAC-Bayes Generalization Bound for Reward Models]
Let $P, Q \in \mathcal{P}(\Phi)$ be any prior and posterior distributions over reward model parameters, and let $\ell$ be a bounded loss function taking values in $[0,1]$. Then, for any $\delta \in (0,1]$, with probability at least $1 - \delta$ over the choice of training set $\mathcal{S} \sim \mathcal{D}^n$, the following inequality holds:
\begin{equation}
    \mathbb{E}_{\phi \sim Q} \left[ \mathcal{L}_{\mathcal{D}}(\phi) \right]
    \;\leq\;
    \mathbb{E}_{\phi \sim Q} \left[ \mathcal{L}_{\mathcal{S}}(\phi) \right]
    +
    \sqrt{\frac{ \mathrm{KL}(Q \| P) + \log\frac{n}{\delta} }{2(n - 1)}}.
\end{equation}
\end{theorem}
\begin{proof}
The subsequent proof follows the classical PAC-Bayes derivation \cite{mcallester1999pac, seeger2002pac} and we prove it again in our scenario. Throughout, we work on the probability space induced by the i.i.d. sample $\mathcal{S}\sim\mathcal{D}^n$ (Assumption~\ref{ass:indist}).

For any measurable function $f\colon\Phi\to\mathbb{R}$ and any
posterior $Q\!\ll\! P$, the \emph{Donsker--Varadhan} variational formula
yields:
\begin{equation}
\label{eq:dv}
  \mathbb{E}_{\phi\sim Q}[f(\phi)]
  \;\le\;
  \frac{1}{\lambda}\!
  \Bigl(
     \log\mathbb{E}_{\phi\sim P}\!\bigl[e^{\lambda f(\phi)}\bigr]
     + \mathrm{KL}(Q\Vert P)
  \Bigr),
  \quad
  \forall\,\lambda>0.
\end{equation}
Fix $\phi\in\Phi$ and let
\begin{equation}
   Z_i := \ell(R_\phi(q_i,h_i),y_i)\in[0,1]
\end{equation}
for $i=1,\dots,n$. By Hoeffding’s inequality,
\begin{equation}
  \mathbb{E}_{\mathcal{S}}\Bigl[%
    \exp\bigl(\lambda(\mathcal{L}_{\mathcal{D}}(\phi)-\mathcal{L}_{\mathcal{S}}(\phi))\bigr)
  \Bigr]
  \;\le\;
  \exp\!\Bigl(\frac{\lambda^{2}}{8n}\Bigr),
  \qquad
  \forall\,\lambda\in\mathbb{R}.
\end{equation}
Taking expectations over $\phi\sim P$ and applying Fubini’s theorem gives
\begin{equation}
\label{eq:mgf}
  \mathbb{E}_{\mathcal{S}}\!
  \Bigl[
    \mathbb{E}_{\phi\sim P}\!
      \bigl[
        e^{\lambda(\mathcal{L}_{\mathcal{D}}(\phi)-\mathcal{L}_{\mathcal{S}}(\phi))}
      \bigr]
  \Bigr]
  \;\le\;
  \exp\!\Bigl(\frac{\lambda^{2}}{8n}\Bigr).
\end{equation}
Define the random variable
\begin{equation}
  \Psi(\mathcal{S})
  \;:=\;
  \log
  \mathbb{E}_{\phi\sim P}\!
    \bigl[
      e^{\lambda(\mathcal{L}_{\mathcal{D}}(\phi)-\mathcal{L}_{\mathcal{S}}(\phi))}
    \bigr]
  -\frac{\lambda^{2}}{8n}.
\end{equation}
By Equation~\ref{eq:mgf}, $\mathbb{E}_{\mathcal{S}}[\exp(\Psi(\mathcal{S}))]\le 1$.
Hence, by Markov’s inequality,
\begin{equation}
  \Pr_{\mathcal{S}}
  \bigl[\Psi(\mathcal{S})>\log\frac{1}{\delta}\bigr]
  \;\le\;
  \delta.
\end{equation}
Thus, with probability at least $1-\delta$ over $\mathcal{S}\sim\mathcal{D}^n$,
\begin{equation}
\label{eq:psi}
  \log
  \mathbb{E}_{\phi\sim P}\!
    \bigl[
      e^{\lambda(\mathcal{L}_{\mathcal{D}}(\phi)-\mathcal{L}_{\mathcal{S}}(\phi))}
    \bigr]
  \;\le\;
  \frac{\lambda^{2}}{8n}
  +\log\frac{1}{\delta}.
\end{equation}
Condition on any $\mathcal{S}$ satisfying Equation~\ref{eq:psi}.
Applying Equation~\ref{eq:dv} with
\begin{equation}
  f(\phi)=\lambda(\mathcal{L}_{\mathcal{D}}(\phi)-\mathcal{L}_{\mathcal{S}}(\phi))
\end{equation}
and the bound in Equation~\ref{eq:psi} yields
\begin{equation}
  \lambda\,
  \mathbb{E}_{\phi\sim Q}\!
     \bigl[\mathcal{L}_{\mathcal{D}}(\phi)-\mathcal{L}_{\mathcal{S}}(\phi)\bigr]
  \;\le\;
  \frac{\lambda^{2}}{8n}
  + \mathrm{KL}(Q\Vert P)
  + \log\frac{1}{\delta}.
\end{equation}
Dividing by $\lambda>0$ and optimizing w.r.t.\ $\lambda$ gives the tightest (sub-Gaussian) bound at
\begin{equation}
  \lambda^{\star}=4\,\bigl(\frac{n-1}{2}\bigr)^{1/2}.
\end{equation}
Plugging $\lambda^{\star}$ back leads to
\begin{equation}
  \mathbb{E}_{\phi\sim Q} [\mathcal{L}_{\mathcal{D}}(\phi)]
  \;\le\;
  \mathbb{E}_{\phi\sim Q} [\mathcal{L}_{\mathcal{S}}(\phi)]
  +
  \sqrt{\frac{\mathrm{KL}(Q\Vert P)+\log\!\frac{n}{\delta}}{2(n-1)}}.
\end{equation}
Since the derivation holds on the
$1-\delta$ event triggered in Equation~\ref{eq:psi},
the bound is valid with the claimed confidence level,
which completes the proof.
\end{proof}

\section{Proof of Theorem~\ref{thm:accuracy}}
\label{app:proof_2}
\begin{theorem}[Answer–Accuracy Bound with Reward‐Gap Parameter]
Let \(q\) be a fixed question, and let $\mathcal{H} = \{h_{1},\dots,h_{N}\} \sim\pi_{\theta}^{\otimes N}$ be \(N\) i.i.d.\ candidate reasoning paths.  Define $h^{*} \;=\;\arg\max_{h\in\mathcal{H}}R^{*}(q,h)$, and $h_{\mathrm{sel}} =\arg\max_{h\in\mathcal{H}}R_{\phi}(q,h)$. The reward‐gap $\gamma(q)=R^{*}\bigl(q,h^{*}\bigr)-\max_{h\in\mathcal{H}\setminus\{h^{*}\}}
       R^{*}(q,h)\ge 0.$ Let
$p_{N,\tau}(q)=\Pr_{\mathcal{H}\sim\pi_{\theta}^{\otimes N}}\bigl[\exists\,h\in\mathcal{H}:\;R^{*}(q,h)\ge\tau\bigr].$
Suppose further that the following hold:
\begin{itemize}
    \item There exist \(\varepsilon\in(0,1]\) and \(\delta\in(0,1)\) such that $\Pr\!\Bigl[\sup_{h\in\mathcal{H}}\bigl|R_{\phi}(q,h)-R^{*}(q,h)\bigr|\le\varepsilon\Bigr]\;\ge\;1-\delta$. And denote the high‐probability event by \(\mathcal{G}=\{\sup_{h}|R_{\phi}-R^{*}|\le\varepsilon\}\).
    \item Conditioned on \(\mathcal{H}\), the deviations \(\Delta_{h}=R_{\phi}(q,h)-R^{*}(q,h)\) are independent, mean‐zero, and satisfy \(|\Delta_{h}|\le\varepsilon\) almost surely.
    \item Assumptions \ref{assump:path-to-answer} and Assumption \ref{assump:coverage} hold.
\end{itemize}
Then the probability of selecting a correct answer satisfies
\begin{equation}
  \Pr\bigl[a(h_{\mathrm{sel}})=a^{*}(q)\bigr]
  \;\ge\;
  p_{N,\tau}(q)\,\Bigl[
    1 \;-\;\delta
        \;-\;(N-1)\,\exp\!\bigl(-\frac{\gamma(q)^{2}}{8\varepsilon^{2}}\bigr)
  \Bigr].
\end{equation}
\end{theorem}

\begin{proof}
Define the failure events
\begin{equation}
  E_{1}
  = \bigl\{\mathcal{H}\cap\{h:R^{*}(q,h)\ge\tau\}=\varnothing\bigr\},
  \quad
  E_{2}
  = \{\,h_{\mathrm{sel}}\neq h^{*}\}.
\end{equation}
By soft‐correctness, success \(\{a(h_{\mathrm{sel}})=a^{*}(q)\}\)
is the complement of \(E_{1}\cup E_{2}\).  We bound:
\begin{equation}
  \Pr(E_{1}) = 1 - p_{N,\tau}(q).
\end{equation}
Condition on \(E_{1}^{c}\) so that \(h^{*}\) exists.  On the event
\(\mathcal{G}\) we have
\begin{equation}
  R_{\phi}(q,h_{\mathrm{sel}})
  \;\ge\;
  R_{\phi}(q,h^{*})
  \quad\Longrightarrow\quad
  R^{*}(q,h_{\mathrm{sel}})
  \;\ge\;
  R^{*}(q,h^{*}) \;-\; 2\varepsilon.
\end{equation}
Hence any competitor \(h\neq h^{*}\) must overcome a gap of at least
\(\gamma(q)-2\varepsilon\).  By Hoeffding’s inequality for the bounded,
independent deviations \(\{\Delta_{h}\}_{h\neq h^{*}}\),
\begin{equation}
  \Pr\bigl[E_{2}\mid E_{1}^{c},\,\mathcal{G}\bigr]
  \;\le\;
  (N-1)\,
  \exp\!\bigl(-\frac{\gamma(q)^{2}}{8\varepsilon^{2}}\bigr).
\end{equation}
Finally, applying the law of total probability and the union bound gives
\begin{align*}
  \Pr(E_{1}\cup E_{2})
  &\le \Pr(E_{1})
     \;+\;\Pr(E_{2}\cap\mathcal{G}\mid E_{1}^{c})\Pr(E_{1}^{c})
     \;+\;\Pr(\mathcal{G}^{c}) \\[4pt]
  &\le (1-p_{N,\tau}(q))
     \;+\;p_{N,\tau}(q)\,(N-1)e^{-\gamma(q)^{2}/8\varepsilon^{2}}
     \;+\;\delta.
\end{align*}
Subtracting from 1 yields the bound
Equation \ref{eq:accuracy‐gap‐bound}.  The asymptotic form for
\(\varepsilon/\gamma(q)\to0\) follows by observing that
\(\exp(-\gamma(q)^{2}/8\varepsilon^{2})\to0\).
\end{proof}

\section{Proof of Corollary~\ref{cor:budget}}
\begin{corollary}[Target Accuracy Constraint on Sampling and Margin]
Given a generalization error bound \(\varepsilon > 0\), a confidence parameter \(\delta\in(0,1)\), and a target answer accuracy level \(\alpha \in (0,1)\). Under the assumptions of Theorem~\ref{thm:accuracy}, if one wishes to guarantee $\Pr[a(h_{\mathrm{sel}})=a^*(q)] \;\ge\; \alpha$, then the sampling coverage probability must satisfy
\begin{equation}
   p_{N,\tau}(q)
   \;\ge\;
   \frac{\alpha}
        {1 - \delta - (N\!-\!1)\,
            \exp\!\left(-\gamma(q)^2\,/\,8\varepsilon^2\right)} .
\end{equation}
\end{corollary}
\begin{proof}
By invoking Theorem~\ref{thm:accuracy} we have
\begin{equation*}
  \Pr\bigl[a(h_{\mathrm{sel}})=a^*(q)\bigr]
  \;\ge\;
  p_{N,\tau}(q)\,\Bigl[\,1 - \delta \;-\;(N-1)\exp\!\Bigl(-\tfrac{\gamma(q)^2}{8\varepsilon^2}\Bigr)\Bigr].
\end{equation*}
In order to guarantee \(\Pr[a(h_{\mathrm{sel}})=a^*(q)]\ge\alpha\), it suffices to enforce
\begin{equation}
  p_{N,\tau}(q)\,\Bigl[\,1 - \delta \;-\;(N-1)e^{-\gamma(q)^2/(8\varepsilon^2)}\Bigr]
  \;\ge\;\alpha.
\end{equation}
Under the standing assumption that  $ 1 - \delta - (N-1)\exp(-\gamma(q)^2/(8\varepsilon^2))>0$, we may divide both sides by this positive quantity, yielding
\begin{equation}
  p_{N,\tau}(q)
  \;\ge\;
  \frac{\alpha}
       {1 - \delta - (N-1)\exp\!\bigl(-\gamma(q)^2/(8\varepsilon^2)\bigr)},  
\end{equation}
which is precisely the bound stated in Equation \ref{eq:alpha-constraint}.
\end{proof}

\section{Comparison with Other Methods}
\label{sec:compare}
In this section, we present a comparison of answer accuracy between CATS and several recent external TTS baselines across three benchmark datasets: GSM8K \cite{cobbe2021gsm8k}, MATH, and OlympiadBench \cite{he2024olympiadbench} in Table~\ref{tab:compare}.
\begin{table}[htb]
    \centering
    \caption{Comparison of answer accuracy with other external TTS methods on the GSM8K, MATH, and OlympiadBench datasets.}
    \resizebox{\linewidth}{!}{
    \begin{tabular}{llccc}
    \toprule
    Method & Base Model & GSM8K & Math & OlympiadBench \\
    \midrule
    STILL-1 \cite{jiang2024technical}        & Llama-3-8B-Instruct & - & - & 34.3 \\
    LiteSearch \cite{wang2024litesearch}     & Llama-3-8B-Instruct & 75.7 & - & - \\
    AlphaMath \cite{chen2024alphamath}      & DeepSeekMath-7B-Base & 83.2 & 64.0 & - \\
    MCTS-DPO \cite{xie2024monte}      & Llama-3.1-8B-Instruct & 85.7 & - & - \\
    NuminaMath-72B-CoT \cite{li2024numinamath} & Qwen2-72B         & 90.8 & 66.7 & 32.6 \\
    LLaMA-Berry \cite{zhang2024llama}    & Llama-3.1-8B-Instruct & 96.1 & 75.3 & 55.1 \\
    MCTSr \cite{zhang2024accessing}          & Llama-3-8B-Instruct   & 96.7 & 58.2 & - \\
    BoostStep \cite{zhang2025booststep} & Qwen2.5-Math-72B-Instruct & - & 85.2 & 52.7 \\
    \midrule
    CATS            & Llama-3.1-8B-Instruct & 97.1 & 76.9 & 56.1 \\
    CATS            & Llama-3.2-1B-Instruct & 88.4 & 61.8 & 33.6 \\ 
    CATS            & Qwen2.5-Instruct-3B & 96.5 & 79.3 & 38.1 \\
    CATS            & Qwen2.5-Instruct-7B &\bf 98.0 &\bf 89.4 &\bf 58.4 \\
    \bottomrule
    \end{tabular}
    }
    \label{tab:compare}
\end{table}
The results show that CATS consistently outperforms prior methods under comparable base models. Notably, CATS surpasses methods based on much larger models such as Qwen2-72B and DeepSeekMath-72B. For instance, on GSM8K, CATS with Qwen2.5-7B attains an accuracy of 98.0\%, exceeding the previous best result of 96.7\% reported by MCTSr. These results demonstrate the effectiveness of our proposed CATS, even with smaller model sizes.

\section{Full Results}
\label{app:full}
The full results of all policy models and PRMs in the MATH-500 dataset are provided in Figure \ref{fig:full_MATH}, and the results of AIME are provided in Figure \ref{fig:full_AIME}.

\begin{figure}[ht]
    \centering
    \includegraphics[width=\linewidth]{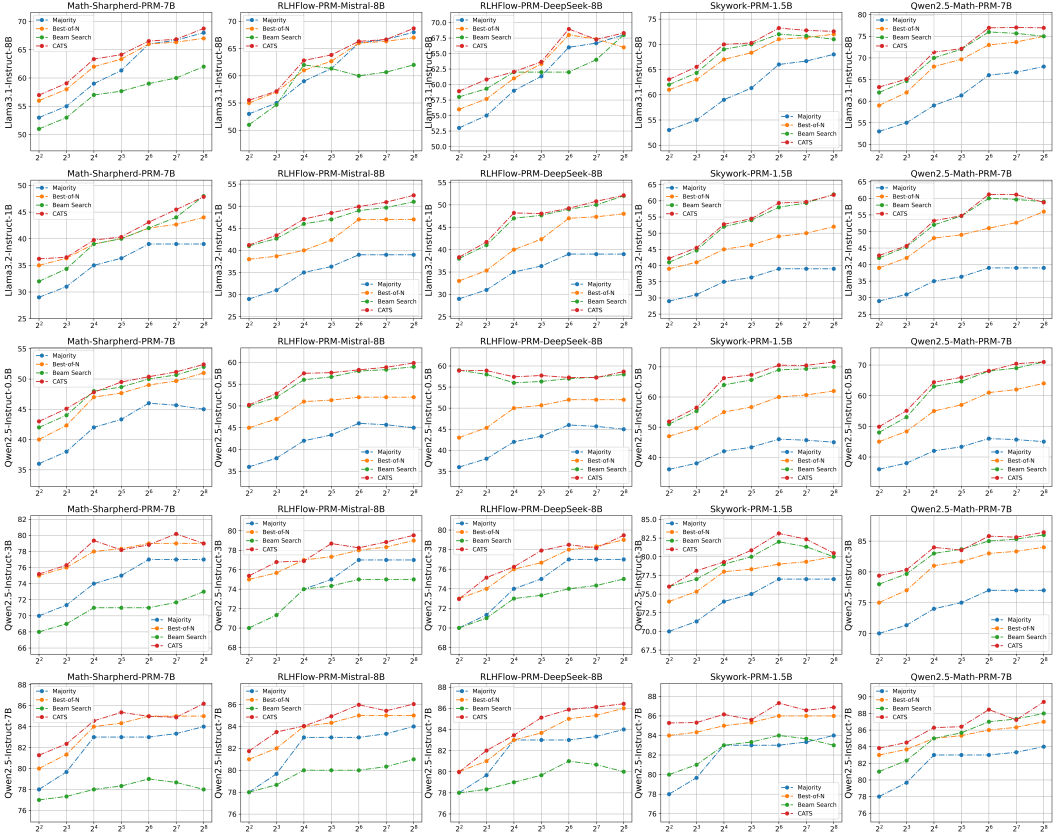}
    \caption{Full results on the MATH500 dataset.}
    \label{fig:full_MATH}
\end{figure}

\begin{figure}[ht]
    \centering
    \includegraphics[width=\linewidth]{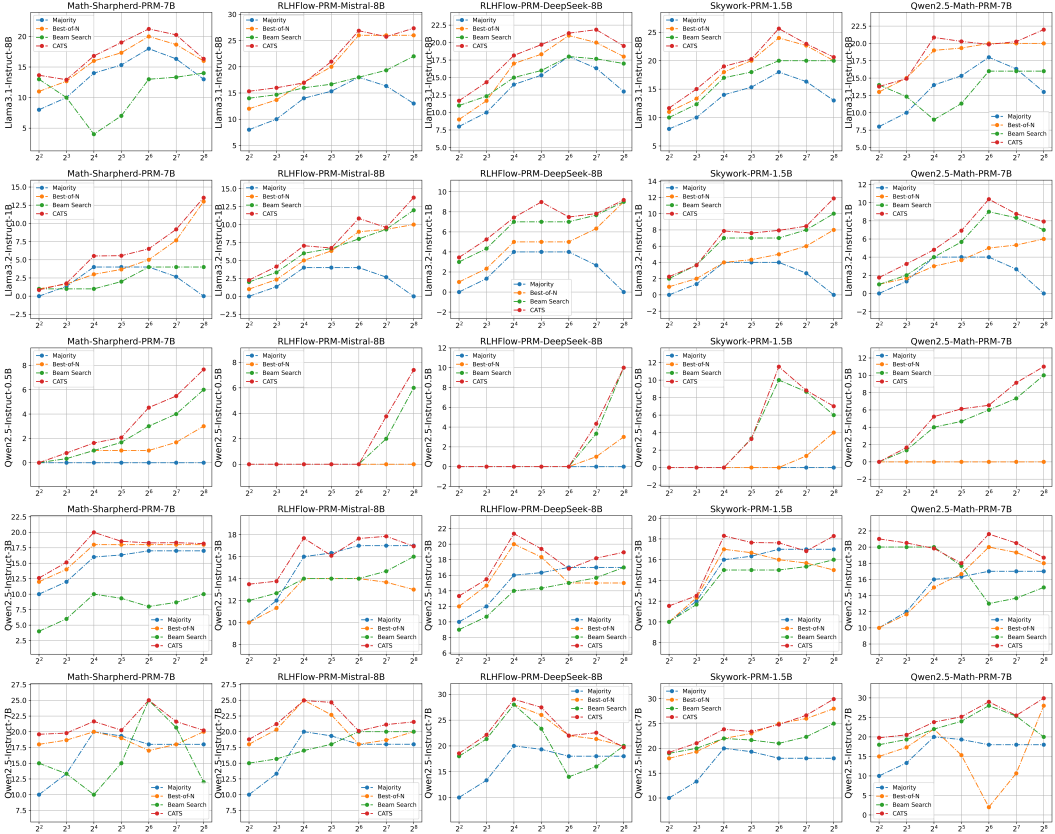}
    \caption{Full results on the AIME24 dataset.}
    \label{fig:full_AIME}
\end{figure}

\section{Analysis of Sparsity in PRM}
\label{app:sparsity}
In this section, we evaluate the validity of using sparsity as a proxy for reward model generalization error. Specifically, we consider two sparsity-based indicators: the overall parameter sparsity of the model and the sparsity of its output layer. The sparsity is calculated by counting the ratio of parameters with values smaller than $1\times 10^{-4}$. To assess generalization performance, we use the test set from the PRM800K dataset, where each example contains a question, a reasoning step, and a binary label indicating the correctness of that step. For each reward model, we compute the cross-entropy between its predicted reward scores and the ground-truth labels across the test set. This deviation reflects the degree of misalignment between predicted and true rewards, and thus serves as an empirical estimate of generalization error. The results is illustrated in Table~\ref{tab:sparse}.
\begin{table}[htb]
    \centering
    \caption{The relationship between sparsity of different PRMs and the test error.}
    \begin{tabular}{lcccc}
    \toprule
    PRM  & \#Params & Total Sparsity & Last layer Sparsity & Test Error \\
    \midrule
    Math-Shepherd-PRM-7B     & 7.11B & 0.0290 & 0.0196  & 2.78 \\
    RLHFlow-PRM-Mistral-8B   & 8.03B & 0.0068 & 0.0060  & 3.87 \\
    RLHFlow-PRM-DeepSeek-8B  & 8.03B & 0.0068 & 0.0060  & 3.87 \\
    Skywork-PRM-1.5B         & 1.54B & 0.0029 & 0.0059  & 4.43 \\
    Qwen2.5-Math-PRM-7B      & 7.08B & 0.0060 & 0.0080  & 3.47 \\
    \bottomrule
    \end{tabular}
    \label{tab:sparse}
\end{table}
As presented in Table~\ref{tab:sparse}, we can observe a clear correlation between model sparsity and generalization behavior, supporting the use of sparsity as a practical and observable proxy in our control framework.

\section{Ablation Study}
\label{app:abl}
\subsection{The Role of Parameter Sparsity in CATS}
To investigate the impact of parameter sparsity in the CATS framework, we conduct an ablation study comparing model performance with and without the parameter sparsity included in the state representation. Specifically, we use Qwen2.5-Instruct-7B as the policy model and evaluate the average performance on the MATH-500 and AIME24 datasets. The results are illustrated in Table \ref{tab:abl_sparse}.
\begin{table}[htb]
    \centering
    \caption{Average accuracy on MATH-500 and AIME24 using Qwen2.5-Instruct-7B, with and without parameter sparsity as part of the state.}
    \begin{tabular}{lcc}
    \toprule
    PRM & MATH-500 & AIME-24  \\
    \midrule
    \multicolumn{3}{c}{\textit{With Parameter Sparsity}} \\
    \midrule
    Math-Shepherd-PRM-7B    & 83.64 & 20.70 \\
    RLHFlow-PRM-Mistral-8B  & 84.05 & 21.13 \\
    RLHFlow-PRM-DeepSeek-8B & 83.61 & 22.27 \\
    Skywork-PRM-1.5B        & 85.98 & 23.55 \\
    Qwen2.5-Math-PRM-7B     & 86.26 & 24.09 \\
    \midrule
    \multicolumn{3}{c}{\textit{Without Parameter Sparsity}} \\
    \midrule
    Math-Shepherd-PRM-7B    & 82.14 & 20.40 \\
    RLHFlow-PRM-Mistral-8B  & 83.26 & 20.76 \\
    RLHFlow-PRM-DeepSeek-8B & 82.55 & 21.14 \\
    Skywork-PRM-1.5B        & 84.47 & 22.45 \\
    Qwen2.5-Math-PRM-7B     & 85.59 & 23.48 \\
    \bottomrule
    \end{tabular}
    \label{tab:abl_sparse}
\end{table}
The result in Table~\ref{tab:abl_sparse} shows that incorporating sparsity leads to improved performance across both datasets for all PRMs, highlighting its effectiveness as a proxy signal for reward model generalization error in CATS.

\subsection{Ablation of hyperparameters}
\paragraph{Ablation of $\lambda_c,\lambda_m,\lambda_r$.}
The hyperparameters $\lambda_c$, $\lambda_m$, and $\lambda_r$ correspond to the coefficients of the compute cost term $C(a_t)$, the margin-based reward difference $\Delta_m(s_t,a_t)$, and the maximum predicted reward $\max_{h\in\mathcal{H}} R_\phi(q,h)$ in the reward function $r(s_t,a_t)$, respectively. We perform an ablation study on the MATH-500 dataset using Qwen2.5-Math-PRM-7B as the policy model. For each configuration, we report the average accuracy across all reward models and compute budgets. The results are summarized in Table~\ref{tab:lambda}.
\begin{table}[htb]
    \centering
    \caption{Mean accuracy on MATH-500 under different combinations of $\lambda_c,\lambda_m,$ and $\lambda_r$}
    \label{tab:lambda}
    \begin{tabular}{ccc|c}
    \toprule
    $\lambda_c$ & $\lambda_m$ & $\lambda_r$ & Accuracy (\%) \\
    \midrule
    0.2 & 0.5 & 0.3 &\bf 84.71 \\
    0.3 & 0.3 & 0.3 & 84.51 \\
    0 & 0.5 & 0.5 & 84.15 \\
    0.5 & 0 & 0.5 & 84.05 \\
    0.5 & 0.5 & 0 & 84.12 \\
    0 & 0 & 1 & 83.46 \\
    0 & 1 & 0 & 83.65 \\
    1 & 0 & 0 & 82.25 \\
    \bottomrule
    \end{tabular}
\end{table}
From the results in Table~\ref{tab:lambda}, we can observe that the best performance is achieved when all three components are present, with moderate weighting ($\lambda_c = 0.2$, $\lambda_m = 0.5$, $\lambda_r = 0.3$). This suggests that each term in the reward function contributes to overall accuracy, and that carefully balancing these terms is essential for optimal performance. We also note that removing any single component leads to a consistent drop in accuracy. In particular, configurations that entirely exclude either $\lambda_m$ or $\lambda_r$ result in performance degradation of over 1\%. This indicates that $\Delta_m(s_t,a_t)$ and $\max_{h\in\mathcal{H}} R_\phi(q,h)$ are both critical. Interestingly, the configuration with only the cost term ($\lambda_c = 1, \lambda_m=\lambda_r=0$) performs the worst, highlighting that $C(a_t)$ alone is insufficient. These findings validate the design of our composite reward function and demonstrate the necessity of jointly modeling compute, ranking confidence, and reward scale.

\paragraph{Ablation of Network Structure}
In this section, we investigate how different architectural choices for the actor and critic networks affect the performance of CATS. We perform an ablation study on the MATH-500 dataset using Qwen2.5-Math-PRM-7B as the policy model. For each configuration, we report the average accuracy across all reward models and compute budgets. Specifically, we vary the number of layers and hidden dimensions of both networks to assess their impact on overall performance.
\begin{table}[t]
\centering
\caption{Mean accuracy on MATH-500 under different Actor and Critic architectures. Each result is averaged over all reward models and the compute budgets.}
\label{tab:actor-critic-ablation}
\begin{tabular}{cc|cc|c}
\toprule
\textbf{Actor Layers} & \textbf{Actor Dim} & \textbf{Critic Layers} & \textbf{Critic Dim} & \textbf{Accuracy (\%)} \\
\midrule
2 & 128 & 2 & 128 & 84.61 \\
2 & 128 & 2 & 256 &\bf 84.71 \\
2 & 256 & 2 & 128 & 84.44 \\
2 & 256 & 2 & 256 & 84.53 \\
2 & 512 & 2 & 512 & 84.23 \\
3 & 128 & 3 & 128 & 84.41 \\
3 & 128 & 3 & 256 & 84.66 \\
3 & 512 & 3 & 512 & 84.03 \\
\bottomrule
\end{tabular}
\end{table}
From Table~\ref{tab:actor-critic-ablation}, we can observe that the best result is achieved when using 2-layer actor and 2-layer critic networks with hidden dimensions of 128 and 256, respectively. Increasing the hidden size beyond 256 or adding more layers does not lead to further improvement and may even result in performance degradation, possibly due to overfitting or optimization instability. These results suggest that lightweight network architectures are sufficient for effective reasoning control in CATS.

\paragraph{Ablation of $\gamma$}
The discount factor $\gamma$ controls the relative importance of long-term versus immediate rewards in the value estimation of the critic. To evaluate its impact, we conduct an ablation study on the MATH-500 dataset using Qwen2.5-Math-PRM-7B as the policy model. We vary $\gamma$ across a range of values and report the average accuracy across all reward models and compute budgets. The results are presented in Table~\ref{tab:gamma-ablation}.
\begin{table}[htb]
\centering
\caption{Mean accuracy (\%) on MATH-500 for different values of the discount factor \(\gamma\), averaged over all reward models and compute budgets.}
\label{tab:gamma-ablation}
\begin{tabular}{c c c c c c c c}
\toprule
\(\gamma\) & 0.5 & 0.7 & 0.9 & 0.95 & 0.99 & 1.0 \\
\midrule
Accuracy (\%) & 84.21 & 84.33 & \textbf{84.71} & 84.63 & 84.60 & 84.56 \\
\bottomrule
\end{tabular}
\end{table}
From the results in Table~\ref{tab:gamma-ablation}, we observe that the choice of the discount factor $\gamma$ has an effect on performance. Accuracy improves as $\gamma$ increases from 0.5 to 0.9, with the best result achieved at $\gamma = 0.9$. This suggests that considering future reward signals over a moderate horizon helps the critic estimate value more effectively. However, further increasing $\gamma$ beyond 0.9 leads to a slight decline in performance. These findings indicate that a moderately high discount factor strikes a good balance between immediate reward and future planning in reasoning control.

\section{Pseudo Code for CATS}
\label{app:pse}
We provide the pseudo code for training and testing the proposed CATS algorithm. The training procedure is detailed in Algorithm~\ref{alg:train}, while the test-time inference procedure is outlined in Algorithm~\ref{alg:test}.

\begin{algorithm}[htb]
\label{alg:train}
\caption{Actor-Critic Training in Compute-Aware Tree Search}
\KwIn{Environment $\mathcal{E}$, PRM $R_\phi$, Actor $\pi_\nu(a \mid s)$, Critic $V_\phi(s)$}
\KwIn{Hyperparameters: learning rate $\eta$, discount factor $\gamma$, beam size $K$, max steps $T$}
\KwResult{Trained actor $\pi_\nu$ and critic $V_\phi$}

Initialize actor $\pi_\nu$ and critic $V_\phi$ with parameters from \texttt{cats\_config}\;

\For{$t = 1$ \KwTo $T$}{
    Reset environment: $(q, a_0) \leftarrow \mathcal{E}.\texttt{reset}()$\;
    Initialize root node $h_0$ with state $s_0 \leftarrow \texttt{ExtractFeatures}(h_0)$\;
    Initialize beam $\mathcal{B}_0 \leftarrow \{h_0\}$\;
    \For{$d = 1$ \KwTo max\_depth}{
        $\mathcal{B}_{d} \leftarrow \emptyset$ \tcp*{next-level beam}
        \ForEach{$h \in \mathcal{B}_{d-1}$}{
            $s_t \leftarrow \texttt{ExtractFeatures}(h)$\;
            Sample $a_t \sim \pi_\nu(\cdot \mid s_t)$ and compute $\log\pi_\nu(a_t \mid s_t)$\;
            Expand node $h$ using action $a_t$, producing and retain children $\{h'_i\}$\;
            Compute reward $r_t \leftarrow \texttt{Reward}(h, \{h'_i\})$\;
            Store $(s_t, a_t, \log\pi_\nu(a_t \mid s_t), r_t)$ in $h$ for each child\;
            Add $\{h'_i\}$ to $\mathcal{B}_{d}$\;
        }
        Prune $\mathcal{B}_{d}$ to Beam Size based on $R_\phi$\;
        \ForEach{$h' \in \mathcal{B}_d$}{
            $s_{t+1} \leftarrow \texttt{ExtractFeatures}(h')$\;
            Retrieve $(s_t, a_t, \log\pi_\nu(a_t \mid s_t), r_t)$ from parent node\;
            Compute TD-error: $\delta_t = r_t + \gamma V_\phi(s_{t+1}) - V_\phi(s_t)$\;
            Update critic: $\phi \leftarrow \phi - \eta \nabla_\phi \left( \frac{1}{2} \delta_t^2 \right)$\;
            Update actor: $\theta \leftarrow \theta + \eta \nabla_\theta \left( \log\pi_\nu(a_t \mid s_t) \cdot \delta_t \right)$\;
        }
    }
}
\end{algorithm}

\begin{algorithm}[htb]
\label{alg:test}
\caption{Inference with CATS}
\KwIn{Environment $\mathcal{E}$, PRM $R_\phi$, Trained Actor $\pi_\nu$, Beam size $K$, Max steps $T$}
\KwResult{Set of completed reasoning paths with associated scores}

Initialize environment: $(q, a_0) \leftarrow \mathcal{E}.\texttt{reset}()$\;
Initialize root node $h_0$ with state $s_0 \leftarrow \texttt{ExtractFeatures}(h_0)$\;
Initialize beam $\mathcal{B}_0 \leftarrow \{(h_0, \mathcal{E})\}$\;
Initialize completed set $\mathcal{F} \leftarrow \emptyset$\;

\For{$t = 1$ \KwTo $T$}{
    $\mathcal{B}_{\text{next}} \leftarrow \emptyset$\;
    \ForEach{$(h, \mathcal{E}_h) \in \mathcal{B}_{t-1}$}{
        $s_t \leftarrow \texttt{ExtractFeatures}(h)$\;
        Sample action $a_t = \arg\max \pi_\nu(s_t)$ \;
        Expand node $h$ using action $a_t$, generating children $\{h'_i\}$\;
        Retain children from $\{h'_i\}$ according to $a_t$\;
        \ForEach{retained $h'_i$}{
            Copy environment $\mathcal{E}' \leftarrow \mathcal{E}_h.\texttt{copy}()$\;
            Step forward: $(a, r, \texttt{done}, \_, \_) \leftarrow \mathcal{E}'.\texttt{step}(h'_i.\texttt{action})$\;
            \eIf{\texttt{done}}{
                Mark $h'_i$ as terminal and add $(h'_i, \mathcal{E}')$ to $\mathcal{F}$\;
            }{
                Further expand $h'_i$ using legal actions from $\mathcal{E}'$\;
                Add $(h'_i, \mathcal{E}')$ to $\mathcal{B}_{\text{next}}$\;
            }
        }
    }
    Prune $\mathcal{B}_{\text{next}}$ to top-$(K - |\mathcal{F}|)$ nodes by score\;
    $\mathcal{B}_t \leftarrow \mathcal{B}_{\text{next}}$\;
    \If{$|\mathcal{F}| = K$}{
        \textbf{break}
    }
}
Return completed set $\mathcal{F}$ as final trajectories\;

\end{algorithm}


\end{document}